\documentclass[journal]{IEEEtran}
\ifCLASSINFOpdf
 \usepackage[pdftex]{graphicx}
\else
\fi
\usepackage{amsmath, amsfonts, amsthm, calc}
\usepackage{algorithmic}

\usepackage{array}

\ifCLASSOPTIONcompsoc
 \usepackage[caption=false,font=normalsize,labelfont=sf,textfont=sf]{subfig}
\else
 \usepackage[caption=false,font=footnotesize]{subfig}
\fi
\usepackage{url}

\usepackage{bm}
\newtheorem{theorem}{Theorem}
\newtheorem{corollary}{Corollary}[theorem]
\newtheorem{lemma}{Lemma}
\newtheorem{proposition}{Proposition}

\newcommand{\E}{\text{E}}
\renewcommand{\P}{\text{P}}

\begin{document}
%
\title{Latent Laplacian Maximum Entropy Discrimination \\
	for Detection of High-Utility Anomalies}
%
%
%

\author{Elizabeth~Hou,~\IEEEmembership{Student Member,~IEEE,}
 Kumar~Sricharan, and~Alfred~O. Hero,~\IEEEmembership{Fellow,~IEEE}
\thanks{Elizabeth Hou and Alfred Hero are with the EECS department at University of Michigan, Ann Arbor, MI, Contact email: emhou@umich.edu }
\thanks{Kumar Sricharan is with PARC, a Xerox company, Palo Alto, CA}}

\maketitle

\begin{abstract}
Data-driven anomaly detection methods suffer from the drawback of detecting all instances that are statistically rare, irrespective of whether the detected instances have real-world significance or not. In this paper, we are interested in the problem of specifically detecting anomalous instances that are known to have high real-world utility, while ignoring the low-utility statistically anomalous instances. To this end, we propose a novel method called Latent Laplacian Maximum Entropy Discrimination (LatLapMED) as a potential solution. This method uses the EM algorithm to simultaneously incorporate the Geometric Entropy Minimization principle for identifying statistical anomalies, and the Maximum Entropy Discrimination principle to incorporate utility labels, in order to detect high-utility anomalies. We apply our method in both simulated and real datasets to demonstrate that it has superior performance over existing alternatives that independently pre-process with unsupervised anomaly detection algorithms before classifying.
\end{abstract}


%
\IEEEpeerreviewmaketitle

\section{Introduction}

Anomaly detection is a very pervasive problem applicable to a variety of domains including network intrusion, fraud detection, and system failures. It is a crucial task in many applications because failure to detect anomalous activity could result in highly undesirable outcomes. For example, (i) detection of anomalous medical claims is important to identify fraud; (ii) detection of fraudulent credit card transactions is necessary to help prevent identity theft; and (iii) detection of abnormal network traffic is necessary to identify hacking. 

Many techniques have been developed for anomaly detection. These methods can be broadly classified into two categories: (i) rule-based systems, and (ii) statistical data-driven approaches. The rule-based systems are based on domain expertise and look for specific types of anomalies while the data-driven approaches look to identify anomalies by identifying statistically rare patterns. Examples of data-driven methods include parametric methods that assume a known family for the nominal (non-anomalous) distribution and non-parametric methods such as those using unsupervised or semi-supervised support vector machines (SVMs) \cite{NIPS1999_1723, gornitz2013toward} or based on minimum volume set estimation \cite{NIPS2006_3145, Scott:2006:LMV:1248547.1248571, NIPS2011_4287}. 

The advantage of data-driven approaches over rule-based methods is that they can identify novel types of anomalies that are unknown to the domain expert. In the network traffic example, they can be used to identify previously unknown types of network attacks that would not have been detected by rule-based systems. The disadvantage is that sometimes the anomalies, while statistically rare, are not interesting to the domain expert. For instance, the data-driven methods would detect routine monthly backup events due to the high volume of network traffic. 

\subsection{Related work}
\label{sec:related-work}
To identify the domain expert's interests, one could simply have the user label instances as high or low utility through active learning frameworks like the algorithms in \cite{pelleg2004active, stokes2008aladin}, and subsequently use popular supervised or semi-supervised classification methods \cite{breiman2001random, Belkin:2006:MRG:1248547.1248632, chapelle2005semi, jebara2004machine} to discriminate between the high-utility and low-utility instances. The drawback with this approach in contrast to our proposed approach is that these methods do not exploit the following key idea: only statistically rare points can be of high- utility, or equivalently, all nominal points are low-utility. As a result, the existing methods are less successful in detecting high utility instances given the limited number of labeled instances.

To incorporate this idea that nominal points are low-utility, one could pre-identify anomalous/nominal points using a statistical anomaly detection method \cite{NIPS1999_1723, gornitz2013toward, NIPS2006_3145, Scott:2006:LMV:1248547.1248571, NIPS2011_4287}, and subsequently use the instances labeled as nominal by the anomaly detection method as additional nominal labels for the classifier \cite{veeramachaneni2016ai}. However, as we demonstrate in our experimental results, this strategy is not optimal because the detected anomalies are independent of the utility labels that are available. In contrast, our algorithm holistically utilizes the labeled information to accurately detect anomalies, and the detected anomalies to improve utility classification.

A similar approach to ours was taken in \cite{das2016incorporating} where the authors also distinguish between high utility anomalies and low utility statistical outliers by incorporating human expert utility labels (which they acquire using an active learning loop). Their algorithm is set up to ensure that the anomaly scores of all labeled anomalies (high-utility) is higher than a threshold, and the scores of all labeled nominals (low-utility) is lower than that threshold. Another related approach is the Bayesian posterior probability model of \cite{5585637}. Their algorithm makes similar assumptions about anomalous points being far away, in distance, from the nominal points.

We construct our model using the Maximum Entropy Discrimination (MED) \cite{NIPS1999_1733} framework, a variant of the classical minimum relative entropy principle, but with a discriminant function in some of the constraints. By choosing different priors, discriminant functions, or constraints, the MED framework can be used for corrupt measurements \cite{DBLP:conf/icassp/XieNH14}, infinite mixture classifiers \cite{zhu2011infinite}, and Markov networks \cite{zhu2009maximum} among other applications. In our case, we choose to add constraints that are hinge loss style discriminant functions with latent variables and a regularizer on the smoothness of the discriminant function.

\subsection{Proposed Work}

In this paper, we develop a novel method called Latent Laplacian Maximum Entropy Discrimination (LatLapMED) which detects high-utility anomalies that are of interest to the domain expert by exploiting the idea that all high-utility points are statistically rare. We are interested in situations where we have data $ \bm{X} $ of sample size $n$, but their labels $ y_i $, which denote high utility $ (y_i = 1) $ or not $ (y_i = -1) $ are only partially observed. Some of the samples $ \bm{X}_i $ are also anomalous with \textit{latent} variables indicating whether they are $ (\eta_i = 1) $ or are not $ (\eta_i = 0) $. Without loss of generality, we assume the labels are observed for the first $ l << n $ points and that the first $a$ points are anomalous (all labeled points are anomalous so $ l \le a << n $). 

By adding constraints to the MED framework to incorporate partially labeled observations, the subsequent decision boundary will be able to separate the high-utility anomalous points from the other points despite this incomplete information. However the nominal distribution is unknown, so one way to identify anomalies is by using the Geometric Entropy Minimization (GEM) principle \cite{NIPS2006_3145, NIPS2011_4287}. This idea of integrating the GEM principle into the MED framework has been previously studied by \cite{DBLP:conf/icassp/XieNH14}, who look at classifying nominal points in a fully supervised setting. In our algorithm, we use exploit the probabilistic nature of the MED framework and solve it with the EM algorithm so that the E-step estimates the latent variables with GEM and the M-Step maximizes over only the anomalous points. 

\subsubsection{Notation}

The dataset is of size $n$ where a sample $ \bm{X}_i \in \mathbb{R}^p$. For notational simplicity, we assume the first $l$ samples are labeled as high utility $ (y_i = 1) $ or not $ (y_i = -1) $ and the first $a$ points are anomalous with indicator variables $ (\eta_i = 1) $ and the rest are not $ (\eta_i = 0) $. We denote $\text{KL}( \cdot || \cdot)$ to be the Kullback-Leibler divergence, $\P(\cdot)$ and $\P_0(\cdot)$ to be a probability density and prior respectively, $\E(\cdot)$ as the expectation of random variables with respect to their distribution, $\mathcal{I}(\cdot)$ to be an indicator function, $ M(\cdot | \cdot) $ to be a discrimination function, $Z(\cdot)$ to be the partition function or normalizing constant, and $|| \cdot ||_2$ and $|| \cdot ||_F$ to be the $\ell_2$ and Frobenius norm respectively. The following are parameters for: the decision boundary $\bm{\Theta} = \{\bm{\theta}, b\}$, the margin of each labeled sample $ \gamma_i $, and the smoothness of the discrimination function $\lambda$. Their corresponding Lagrange multipliers are $\alpha_i$ and $\beta$. We define the following matrices: $\bm{I}$ as the identity, $\bm{0}$ as a zero vector, $\mathcal{L}$ as the normalized graph Laplacian matrix, $\bm{K}$ as the Gram matrix of a kernel function $k(\cdot, \cdot)$, $\bm{Y}$ as a diagonal matrix of the labels, $\bm{J}$ as a 0-1 expansion matrix, and $\bm{H}$ as a diagonal matrix of the anomaly indicators with $\bm{h}$ as only its non-zero rows. Anything with a ``hat" $\hat{\,\,}$ is an estimator of its true value which has the same symbol, but no ``hat".

The rest of this paper is organized as follows: Section 2 will briefly review the MED framework and discuss constructing maximum margin classifiers with it. Section 3 will propose an additional constraint to incorporate unlabeled points and derive a probabilistic interpretation of the Laplacian SVM. Section 4 will describe the proposed Latent Laplacian MED method, which uses the EM algorithm to simultaneously estimate unobserved anomalous labels and form a utility decision boundary. Section 5 contains simulation results of the performance of our proposed method, an application to a dataset of Reddit subforums, and two applications to datasets of botnet traffic (CTU-13).

\section{Maximum Entropy Discrimination}

Maximum entropy is a classical method of estimating an unknown distribution subject to the expected values of a set of constraints where the expectation is with respect to the unknown distribution. When the prior distribution is not uniform, this can be generalized as minimizing the relative entropy (or Kullback-Leibler divergence). The MED framework \cite{NIPS1999_1733} extends the minimum relative entropy principle to have discriminant power by requiring one of the constraints to be over a parametric family of decision boundaries $ M(\bm{X} | \bm{\Theta}) $. Thus, it creates models that have both the classification robustness of discriminative approaches and the ability to deal with uncertain or incomplete observations of generative approaches.

The basic MED objective function is
\begin{flalign*} 
&\underset{\P(\bm{\Theta}, \bm{\gamma} | \bm{X}, \bm{y})}{\min} \text{KL}\left(\P(\bm{\Theta}, \bm{\gamma} | \bm{X}, \bm{y}) || \P_0(\bm{\Theta}, \bm{\gamma}) \right) \notag \\
& \mathrel{\makebox[\linewidth-3cm]{\text{subject to} }} \notag \\
&\iint \P(\bm{\Theta}, \bm{\gamma}) \, (y_1 M(\bm{X}_1| \bm{\Theta}) - \gamma_1) \, d\bm{\Theta} d\bm{\gamma} \ge 0 \notag \\
&\mathrel{\makebox[\linewidth-3cm]{\vdots}} \\
&\iint \P(\bm{\Theta}, \bm{\gamma}) \, (y_n M(\bm{X}_n| \bm{\Theta}) - \gamma_n) \, d\bm{\Theta} d\bm{\gamma} \ge 0 \notag 
\end{flalign*}
which has solution,
$$ \P(\bm{\Theta}, \bm{\gamma}| \bm{X}, \bm{y}) = \frac{\P_0(\bm{\Theta}, \bm{\gamma})}{Z(\bm{\alpha})} \exp \left\{ \sum_{i=1}^n \alpha_i\left (y_i M(\bm{X}_i | \bm{\Theta}) - \gamma_i \right) \right\} $$
where the rows $ \bm{X}_i \in \mathbb{R}^p $ are samples, $ y_i \in \{-1, 1 \} $ are labels, $ \P_0(\bm{\Theta}, \bm{\gamma}) $ is the joint prior, and $ \bm{\alpha} = [ \alpha_1, ..., \alpha_n ]^T \ge 0 $ are Lagrange multipliers, which can be found by maximizing the negative log partition function $ - \log \left(Z(\bm{\alpha}) \right) $. Because the posterior distribution $ \P(\bm{\Theta}, \bm{\gamma} | \bm{X}, \bm{y}) $ is over the decision and margin parameters $ \bm{\Theta} \text{ and } \bm{\gamma} $, the MED framework gives a distribution of solutions. This gives additional flexibility because the decision rule $ \hat{y}_{i'} = \text{sign}( \iint \P(\bm{\Theta}, \bm{\gamma}| \bm{X}, \bm{y}) M(\bm{X}_{i'} | \bm{\Theta}) d\bm{\gamma} d\bm{\Theta} ) $ is a weighted combination of discriminant functions, and different priors on $ \bm{\gamma} $ can permit different degrees of separability in the classification. If the support of this prior includes negative values, the decision boundary can be found on non-separable data. 

\subsection{Interpretation as a Maximum Margin Classifier} \label{svm}

Specifically in the case when the discriminant function $ M(\bm{X} | \bm{\theta}, b) = \bm{X \theta} + b $ is linear, and the prior distribution is $ \P_0(\bm{\Theta}, \bm{\gamma}) = \P_0(\bm{\theta}) \P_0(b) \prod_{i=1}^n \P_0(\gamma_i ) $ where $ \P_0(\gamma_i) = C e^{-C(1-\gamma_i)} \mathcal{I}(\gamma_i \le 1) $, $ \P_0(\bm{\theta}) \text{ is } N(\bm{0}, \textbf{\emph{I}}) $, and $ \P_0(b) $ is a Gaussian non-informative prior, \cite{NIPS1999_1733} shows that the MED solution is very similar to a support vector machine (SVM). The \textit{maximum a posteriori} (MAP) estimator for $\bm{\theta} $ is $ \sum_{i=1}^n \alpha_i y_i \bm{X}^T_i $ where $ \alpha_i $ maximize $- \log \left(Z(\bm{\alpha}) \right) = $
\begin{flalign*} 
& -\frac{1}{2} \sum_{i=1}^n \sum_{i'=1}^n \alpha_i \alpha_{i'} y_i y_{i'} \bm{X}_i\bm{X}_{i'}^T + \sum_{i=1}^n \left(\alpha_i + \log(1-\frac{\alpha_i }{C}) \right) \\
& \text{subject to } \sum_{i=1}^n y_i \alpha_i = 0 \text{ and } \alpha_i \ge 0 \text{ for all } i 
\end{flalign*}
which has a log barrier term $ \log(1-\alpha_i/C) $ instead of the inequality constraints $ \alpha_i \leq C $ found in the dual form of an SVM. Otherwise the two objective functions are equivalent, so the $ \hat{\alpha}_i $ are roughly the optimal support vectors and would only differ from actual support vectors when the posterior mode lies near the boundary of its support.

The connection between SVMs and Gaussian process classification has been previously studied in many works including \cite{wahba1999support, jaakkola1999probabilistic, smola1998connection, opperwinther, Sollich2002}. The model in \cite{Sollich2002} is also a probabilistic interpretation of an SVM and also uses a MAP estimator with a Gaussian process prior. However, the MED framework is more generalizable and intuitive because we can easily tailor the posterior to have specific properties by narrowing the feasible set of posteriors through additional goodness-of-fit constraints expressed as statistical expectations of fitting errors. In the following sections we will show how the probabilistic interpretation of an SVM can incorporate partially labeled points and latent variables through additional constraints.

\section{MED with Partially Labeled Observations}

In order to incorporate unlabeled points, we use the semi-supervised framework of \cite{Belkin:2006:MRG:1248547.1248632}, which requires the decision boundary to be smooth with respect to the marginal distribution of all the data, $ \mathcal{P}_X $. This is because we assume that unlabeled points have the same label as their labeled neighbors and prefer decision boundaries in low density regions. So we can restrict the choice of posteriors to be one that induces a decision boundary with at least a certain level of expected smoothness by the additional constraint
\begin{flalign*}
\iint \P(\bm{\theta}, \lambda) \left( \int_{x \in \mathcal{M}} || \nabla_{\mathcal{M}} M(\bm{X} | \bm{\theta}) ||_2^2 \, d\mathcal{P}_X - \lambda \right) \, d\bm{\theta} d\lambda \leq 0
\end{flalign*}
where $ \mathcal{M} = \text{supp}(\mathcal{P}_X ) \subset \mathbb{R}^n $ is a compact submanifold, $ \nabla_{\mathcal{M}} $ is the gradient along it, and $ \lambda $ controls the complexity of the decision boundary in the intrinsic geometry of $ \mathcal{P}_X $. Note the bias/intercept term $ b $ does not appear in the constraint.

Since the marginal distribution of the data is unknown, we must approximate the constraint. From \cite{grigor2006heat}, 
\begin{flalign*}
& M(\bm{X} | \bm{\theta})^T \mathcal{L} M(\bm{X} | \bm{\theta}) \rightarrow \int_{x \in \mathcal{M}} || \nabla_{\mathcal{M}} M(\bm{X} | \bm{\theta}) ||_2^2 \, d\mathcal{P}_X
\end{flalign*}
where $ \mathcal{L}$ is the normalized graph Laplacian formed with a heat kernel using all the data. Thus, we define the empirical objective function for this semi-supervised problem as
\begin{flalign*}
& \underset{P(\bm{\theta}, b, \bm{\gamma}, \lambda | \bm{X}, \bm{y})}{\min} \text{KL}\left(\P(\bm{\theta}, b, \bm{\gamma}, \lambda | \bm{X}, \bm{y}) || \P_0(\bm{\theta}, b, \bm{\gamma}, \lambda) \right) \\
& \mathrel{\makebox[\linewidth-3cm]{\text{subject to} }} \\
& \iiint \P(\bm{\theta}, b, \bm{\gamma}) \, (y_i M(\bm{X}_1 | \bm{\theta}, b) - \gamma_1 ) \, d\bm{\theta} db d\bm{\gamma} \ge 0 \notag \\
& \mathrel{\makebox[\linewidth-3cm]{\vdots}} \\
& \iiint \P(\bm{\theta}, b, \bm{\gamma}) \, (y_l M(\bm{X}_l | \bm{\theta}, b) - \gamma_l ) \, d\bm{\theta} db d\bm{\gamma} \ge 0 \notag \\
& \iint \P(\bm{\theta}, \lambda) ( M(\bm{X} | \bm{\theta})^T \mathcal{L}M(\bm{X} | \bm{\theta}) - \lambda ) \, d\bm{\theta} d\lambda \le 0 
\end{flalign*}
which has solution,
\begin{flalign*} 
& \P(\bm{\theta}, b, \bm{\gamma}, \lambda | \bm{X}, \bm{y}) = \frac{\P_0(\bm{\theta}, b, \bm{\gamma}, \lambda)}{Z(\bm{\alpha}, \beta)} \exp \Bigg\{ \\
& \sum_{i=1}^l \alpha_i \left( y_i M(\bm{X}_i | \bm{\theta}, b) - \gamma_i \right) + \beta \left(\lambda - M(\bm{X} | \bm{\theta})^T \mathcal{L} M(\bm{X} | \bm{\theta}) \right) \Bigg\} 
\end{flalign*}
where the $\alpha_i$'s are Lagrange multipliers for the mean goodness-of-fit constraint on $\mathcal M(\bm{X} | \bm{\theta})$ and $ \beta \ge 0 $ is the Lagrange multiplier for the smoothness constraint on $ M(\bm{X} | \bm{\theta}) $. 

\subsection{Laplacian MED as a Maximum Margin Classifier} \label{lapMED}

When one uses a linear discriminant function, the same independent priors as in Section \ref{svm}, but with additional exponential non-informative prior $ \P_0(\lambda) $, the MAP estimator is thus a maximum margin classifier. This estimator is defined as $ \hat{\bm{\theta}} = \sum_{i=1}^l (\bm{I} + 2 \beta \bm{X}^T \bm{L X})^{-1} \bm{X}_i^T y_i \alpha_{i} $ where the Lagrange multipliers $ \bm{\alpha}, \beta$ maximize the negative log partition function $ -\log\left( Z(\bm{\alpha}, \beta) \right) =$
\begin{flalign} \label{lapsvm}
& - \frac{1}{2} \sum_{i=1}^l \sum_{i'=1}^l \alpha_i \alpha_{i'} y_i y_{i'} \bm{X}_i (\textbf{\emph{I}} + 2 \beta \bm{X}^T \bm{L X})^{-1} \bm{X}^T_{i'} \\
& + \sum_{i=1}^l \left(\alpha_i + \log(1-\alpha_i/C) \right) + \log \left(\det(\textbf{\emph{I}} + 2\beta \bm{X}^T \mathcal{L}\bm{X}) \right) \notag 
\end{flalign}
subject to $\sum_{i=1}^l y_i \alpha_i = 0, \alpha_1, \dots, \alpha_l \ge 0, $ and $\beta \ge 0$.
 
Since the smoothness constraint is formulated using the semi-supervised framework of \cite{Belkin:2006:MRG:1248547.1248632}, the above objective function is very similar to their proposed Laplacian SVM (LapSVM). This is more obviously seen by extending \eqref{lapsvm} to nonlinear discriminant functions though a kernel function $ k(\cdot, \cdot) $ and treating $ \beta $ as a fixed parameter to be chosen separately. 

\begin{proposition} \label{lapmed}
Let $ M(\bm{X} | \bm{\theta}, b) = \bm{X \theta} + b $ and $ \P_0(\bm{\theta}, b, \bm{\gamma}, \lambda) = \P_0(\bm{\theta}) \P_0(b) \prod_{i=1}^l \P_0(\gamma_i ) \P_0(\lambda) $ where $ \P_0(\gamma_i) = C e^{-C(1-\gamma_i)} \mathcal{I}(\gamma_i \le 1) $, $ \P_0(\bm{\theta}) \text{ is } N(\bm{0}, \bm{I}) $, and $ \P_0(\lambda) \text{ and } \P_0(b) $ approach exponential and Gaussian non-informative priors. Then for a given parameter $ \beta \ge 0$, the dual problem to maximizing the posterior $ \P(\bm{\theta}, b, \bm{\gamma}, \lambda | \bm{X}, \bm{y}) $ for $ \bm{\theta} $ is
\begin{flalign*}
& \underset{\bm{\alpha}}{\arg\max} \, \sum_{i=1}^l \alpha_i - \frac{1}{2} \bm{\alpha}^T \bm{Y J} \bm{K} ( \bm{I} + 2\beta \mathcal{L}\bm{K})^{-1} \bm{J}^T \bm{Y \alpha} \\
& + \sum_{i=1}^l \log(1-\alpha_i /C) \quad \text{ s.t. } \sum_{i=1}^l y_i \alpha_i = 0, \alpha_1, \dots, \alpha_l \ge 0
\end{flalign*}
where $\bm{K}$ is the Gram matrix of the kernel function, $ \bm{Y} = \emph{diag}(y_1 ,..., y_l) $ and $ \bm{J} = [ \bm{I} \,\, \bm{0} ] $ is a $ l \times n $ expansion matrix. The decision rule in this dual form is $ \hat{y}_{i'} = \emph{sign}\big( k(\bm{X}_{i'}, \bm{X}) ( \bm{I} + 2\beta \mathcal{L}\bm{K})^{-1} \bm{J}^T \bm{Y} \hat{\bm{\alpha}} + \hat{b} \big) $ where $ \hat{b} = \underset{b}{\arg\min} \, \sum_{s \in \{i | \hat{\alpha}_i \neq 0 \} } |(y_s - \hat{y}_s) - b| $ is equivalent to an SVM bias term \cite{Sindhwani:2005:BPC:1102351.1102455}. 
\end{proposition} 

Again the log barrier term produces a relaxation of the inequality constraints $ \alpha_i \leq C $ and decreases the objective function if the optimum is near the boundary of the support. The parameters can be written as $ C = \frac{1}{2 l \gamma_A } $ and $ \beta = \frac{\gamma_I }{2 \gamma_A n^2} $ so that they are functions of $ \gamma_A$ and $\gamma_I$, the penalty parameters in the LapSVM for the norms associated with the reproducing kernel Hilbert space (RKHS) and data distribution $ \mathcal{P}_X $ respectively. Due to these similarities, we will call the classifier of Proposition \ref{lapmed} the Laplacian MED (LapMED).

\section{MED with Latent Variables} 

Now that we have established a method to incorporate unlabeled points in MED, we will present a method to also incorporate latent variables. This joint method of simultaneously incorporating unlabeled points and latent variables is our proposed Latent Laplacian MED (LatLapMED) method. We will first consider the case where we can observe the latent variables, so that we have a complete posterior distribution. Then we will derive a lower bound for the observed posterior distribution and discuss how to deal with estimating the latent variables when they are not observed. This will allow us to apply the EM algorithm, which alternates between estimating the latent variables and maximizing the lower bound.

\subsection{The Complete Posterior} \label{sec:complete_post}

If we observe the anomaly indicator variables $ \eta_i $, then we can construct a posterior that depends on these variables by modifying the constraints on mean goodness-of-fit. The discriminant function $ M(\bm{X}_i , \eta_i | \bm{\theta}, b) = \eta_i (\bm{X}_i \bm{\theta} + b)$ can be used to create a maximum margin classifier that gives positive or negative values for anomalous points and zeros for nominal points. This is reasonable because all labeled points are anomalous, so if they are mistakenly classified as nominal, the loss function embedded in the constraints
\begin{flalign} \label{decision}
& \iiint \P(\bm{\theta}, b, \bm{\gamma}) \, (y_i \eta_1 (\bm{X}_i \bm{\theta} + b) - \gamma_1) \, d\bm{\theta} db d\bm{\gamma} \ge 0 \notag \\
& \mathrel{\makebox[\linewidth-3cm]{\vdots}} \\
& \iiint \P(\bm{\theta}, b, \bm{\gamma}) \, (y_l \eta_l (\bm{X}_l \bm{\theta} + b) - \gamma_l ) \, d\bm{\theta} db d\bm{\gamma} \ge 0 \notag 
\end{flalign}
will penalize the labeled points as if they were inside the margin. 

Additionally if some of the anomalous points are not labeled, then we will use the same semi-supervised framework as before and add a smoothness constraint. Since the discriminant function $ M(\bm{X}_i , \eta_i | \bm{\theta}, b) $ will always give zeros for nominal points, it really only needs to be smooth with respect to the marginal distribution of the anomalies $\mathcal{P}_{X_{\eta}}$. Thus because $ \int_{x \in \mathcal{M}} || \nabla_{\mathcal{M}} M(\bm{X} | \bm{\theta}) ||_2^2 \, d\mathcal{P}_X = \int_{x \in \mathcal{M}_{\bm{\eta}}} || \nabla_{\mathcal{M}_{\eta}} M(\bm{X} | \bm{\theta}) ||_2^2 \, d\mathcal{P}_{X_{\eta}}$, there are two choices for the empirical smoothness constraint that converge to the same limit,
\begin{flalign}
& \iint \P(\bm{\theta}, \lambda) ( \bm{\theta}^T \bm{X}^T \bm{H}^T \mathcal{L}\bm{H} \bm{X} \bm{\theta} - \lambda ) \, d\bm{\theta} d\lambda \le 0 \label{smooth_H} \\
& \iint \P(\bm{\theta}, \lambda) ( \bm{\theta}^T \bm{X}^T \bm{h}^T \mathcal{L}_\eta \bm{h} \bm{X} \bm{\theta} - \lambda ) \, d\bm{\theta} d\lambda \le 0 \label{smooth_h}
\end{flalign}
where $ \mathcal{L}_\eta $ is the normalized graph Laplacian of the anomalous points, $\bm{H} = \text{diag}(\bm{\eta})$, and $ \bm{h} $ is a $ a \times n $ submatrix of only the non-zero rows of $ \bm{H} $. 

The solution to the MED problem, using constraints \eqref{decision} and either \eqref{smooth_H} or \eqref{smooth_h}, is a posterior distribution $ \P(\bm{\theta}, b, \bm{\gamma}, \lambda | \bm{X}, \bm{\eta}, \bm{y}) $ and its MAP estimator can also be a maximum margin classifier when the priors are the ones in Proposition \ref{lapmed}. If possible, it is more ideal to use constraint \eqref{smooth_h} because the maximum margin classifier forms a decision boundary with just the $a$ anomalous points; so it takes considerable less computation time than the equivalent classifier using constraint \eqref{smooth_H}. 

\begin{lemma} \label{complete_post}
	Using the same priors as in Proposition \ref{lapmed}, but with discriminant function $ M(\bm{X}_i , \eta_i | \bm{\theta}, b) = \eta_i (\bm{X}_i \bm{\theta} + b)$, the dual problem to maximizing the posterior of the MED problem with constraints \eqref{decision} and \eqref{smooth_h} is maximizing 
\begin{flalign*}
& \underset{\bm{\alpha}}{\arg\max} \, \sum_{i=1}^l \alpha_i -\frac{1}{2} \bm{\alpha}^T \bm{Y J} \bm{K}_\eta( \bm{I} + 2\beta \mathcal{L}_\eta \bm{K}_\eta)^{-1} \bm{J}^T \bm{Y \alpha} \\
& + \sum_{i=1}^l \log(1-\alpha_i / C) \,\, \text{ s.t. } \sum_{i=1}^l \alpha_i y_i = 0, \alpha_1, \dots, \alpha_l \ge 0
\end{flalign*}
where $ \bm{K}_\eta $ is an $ a \times a $ submatrix of the Kernel matrix and $ \bm{J} = [ \bm{I} \,\, \bm{0} ] $ is now a $ l \times a $ expansion matrix.
\end{lemma}

Now, we consider the more realistic scenario where the anomaly indicator variables are latent. Because $ \mathcal{L}_\eta $ depends on both $\bm{X}$ and $\bm{\eta} $, it is simpler to use constraint \eqref{smooth_H} to derive a posterior. Additionally, the posterior distribution is no longer concave and thus difficult to maximize so we will derive a lower bound to maximize instead.

\subsection{A Lower Bound} 

Since the anomaly indicator variables $\eta_i$ are not actually observable, the posterior distribution we can observe is of the form 
\begin{flalign} \label{obvs_post}
\hspace{-.25cm} \P(\bm{\theta}, b, \bm{\gamma}, \lambda | \bm{X}, \bm{y}) = \frac{ \P_0(\bm{\theta}, b, \bm{\gamma}, \lambda) \sum \P(\bm{X}, \bm{\eta}, \bm{y}| \bm{\theta}, b, \bm{\gamma}, \lambda)}{\sum \P(\bm{X}, \bm{\eta} , \bm{y} | \bm{\alpha}) } \hspace{-.1cm} 
\end{flalign}
where the summation $\sum$ is over all $\eta_i \in \{0, 1\}$. So, we need a lower bound for the negative log expected partition function $ -\log \left( \sum \P(\bm{X}, \bm{\eta}, \bm{y} | \bm{\alpha}) \right) $ that is practical to maximize. 

\begin{lemma} \label{lowerbound}
Let \eqref{obvs_post} be the posterior of the MED problem with constraints \eqref{decision} and \eqref{smooth_H}, then using the same assumptions as Lemma \ref{complete_post}, the dual problem to MAP estimation is maximizing $ -\log \left( \sum \P(\bm{X}, \bm{\eta}, \bm{y} | \bm{\alpha}) \right) $ for $ \bm{\alpha} $. This objective has a lower bound proportional to
\begin{flalign*} 
& \sum_{i=1}^l \alpha_i -\frac{1}{2} \bm{\alpha}^T \bm{Y J} \bm{K} \left( \bm{I} + 2 \beta \E_{\eta} ( \bm{H}^T \mathcal{L}\bm{H}) \bm{K} \right)^{-1} \bm{J}^T \bm{Y \alpha} \\
& + \sum_{i=1}^l \log(1-\alpha_i / C) \,\, \text{ s.t. } \sum_{i=1}^l \alpha_i y_i = 0, \alpha_1, \dots, \alpha_l \ge 0
\end{flalign*}
where $\E_\eta$ is the expectation with respect to $ \P(\bm{\eta} | \bm{X}, \bm{y}, \bm{\alpha}^{t-1}) $ and $ \bm{\alpha}^{t-1} $ are the optimal Lagrange multipliers of the previous iteration.
\end{lemma}

With this lower bound, we have an objective to maximize in the M-step of the EM algorithm. In the following subsection, we give a way to estimate $\E ( \bm{H}^T \mathcal{L}\bm{H} | \bm{X}, \bm{y}, \bm{\alpha}^{t-1} ) = \mathcal{L}\odot \E (\bm{\eta} \bm{\eta}^T | \bm{X}, \bm{y}, \bm{\alpha}^{t-1})$ for the E-step.

\subsection{Estimating the Latent Variables} \label{GEM}

Since $ \eta_i = 1 $ when the data point $ \bm{X}_i $ does not come from the nominal distribution, we can define it as an indicator variable $ \eta_i = \mathcal{I}(\bm{X}_i \notin \Omega_\phi)$ where $ \Omega_\phi $ is a minimum entropy set of level $ \phi $. So $ \eta_i $ can be viewed as the test function for a statistical test of whether the density of $ \bm{X}_i $ is equal to the density of the nominal points or not, and $ \Omega_\phi $ is the optimal acceptance region of the test. However because the nominal distribution of $ \bm{X} $ is unknown, the GEM principle \cite{NIPS2006_3145, NIPS2011_4287} estimates the optimal acceptance region using the property that if $ \underset{K, N \rightarrow \infty}{\lim} \, \frac{K}{N} = \phi $, a greedy K point k nearest neighbors graph (K-kNNG) converges almost surely to the minimum $\upsilon$-entropy set containing at least ($1 - \phi$)\% of the mass. Thus, for any $ ij $ element of the matrix $ \E(\bm{\eta} \bm{\eta}^T | \bm{X}, \bm{y}, \bm{\alpha}^{t-1}) $, we have
\begin{flalign*} 
& \E(\eta_i \eta_j | \bm{X}, \bm{y}, \bm{\alpha}^{t-1}) = \E \left(\mathcal{I}(\bm{X}_i, \bm{X}_j \notin \Omega_\phi) | \bm{X}, \bm{y}, \bm{\alpha}^{t-1} \right) \\
& \approx \mathcal{I}(\bm{X}_i, \bm{X}_j \notin \hat{\Omega}_\phi) = \hat{\eta}_i \hat{\eta}_j 
\end{flalign*}
where $ \hat{\Omega}_\phi $ is the estimated acceptance region.

However, if $ \hat{\Omega}_\phi $ uses the standard K-kNNG with edge lengths equal to Euclidean distances, the graph does not incorporate label information or how the points lie relative to the decision boundary. Since the neighbors of an anomalous point are also most likely anomalous, we instead use a similarity metric that penalizes a point for having anomalous neighbors. So the edge length between a point $i$ and its neighbor $j$ is
\begin{flalign} \label{sim_metric}
& \hspace{-.15cm} |e_{i(j)}| = 
\begin{cases} 
|| \bm{X}_i - \bm{X}_j||_2 + \hat{d}_j^{\, t-1} & \text{ if } \hat{d}_{j}^{t-1} > \rho \text{ or } y_j = 1 \\
|| \bm{X}_i - \bm{X}_j||_2 & \text{ otherwise}
\end{cases} 
\hspace{-.06cm}
\end{flalign}
where $ \hat{d}_j^{t-1} $ is the signed perpendicular distance between $ \bm{X}_j $ and the decision boundary, $\rho \ge 0$ is some threshold, and $ y_j $ is the label of $ \bm{X}_j $. Using a graph with the above edges in the GEM principle, we can estimate the optimal acceptance region, given a decision boundary and labels, by
\begin{flalign} \label{accept_region}
& \hat{\Omega}_\phi = \underset{\mathcal{X}_{N,K} \subset \mathcal{X}_N }{\arg\min} \sum_{i=1}^{K} \sum_{j=1}^k |e_{i(j)}|
\end{flalign}
where $ \mathcal{X}_{N, K} $ is a size $ K $ subset of the set of all points $ \mathcal{X}_N $ and $ \{ e_{i(1)}, ... , e_{i(k)} \} $ are the edges between point $i$ and its $k$ neighbors.

So using the GEM principle described above, $ \mathcal{L}\odot \hat{\bm{\eta}} \hat{\bm{\eta}}^T= \hat{\bm{H}}^T \mathcal{L}\hat{\bm{H}} $ is an estimator for $\E( \bm{H}^T \mathcal{L}\bm{H} | \bm{X}, \bm{y}, \bm{\alpha}^{t-1} )$. However, if the MED problem uses constraint \eqref{smooth_h}, the E-step would need an estimator for $\E( \bm{h}^T \mathcal{L}_\eta \bm{h} | \bm{X}, \bm{y}, \bm{\alpha}^{t-1} ) $ instead. 

\begin{lemma} \label{GEM_est}
Assume that $ \mathcal{L}\odot \hat{\bm{\eta}} \hat{\bm{\eta}}^T= \hat{\bm{H}}^T \mathcal{L}\hat{\bm{H}} $ is a good estimator for $\E( \bm{H}^T \mathcal{L}\bm{H} | \bm{X}, \bm{y}, \bm{\alpha}^{t-1} )$ and that the first $m$ neighbors of any anomalous points are also anomalous. Then $\hat{\bm{h}}^T \hat{\mathcal{L}}_\eta \hat{\bm{h}} $ is a good estimator for $\E( \bm{h}^T \mathcal{L}_\eta \bm{h} | \bm{X}, \bm{y}, \bm{\alpha}^{t-1} ) $ where $ \hat{\bm{h}} $ is the $ a \times n $ submatrix of the nonzero rows of $ \hat{\bm{H}} $ and $ \hat{\mathcal{L}}_{\eta} $ is the Laplacian matrix on only the set of data points $ \{ \bm{X}_i : \hat{\eta}_i = 1 \} $.
\end{lemma}

\subsection{Maximum Margin Classification with the EM Algorithm}

From the previous three subsections, it is obvious that the EM algorithm for MAP estimation of the unobserved posterior distribution $ \P(\bm{\theta}, b, \bm{\gamma}, \lambda | \bm{X}, \bm{y}) $ is also a maximum margin classifier, which we call Latent Laplacian MED (LatLapMED).

\begin{theorem} \label{EM_alg}
Under Lemmas \ref{complete_post} and \ref{GEM_est}, the E-step of the EM algorithm is just getting estimators $ \hat{\eta}_i = \mathcal{I}(\bm{X}_i \notin \hat{\Omega}_\phi) $ for the function of unknown parameters $ \E( \eta_i | \bm{X}, \bm{y}, \bm{\alpha}^{t-1}) $. And, the M-step for maximizing $ \P(\bm{\theta}, b, \bm{\gamma}, \lambda | \bm{X}, \bm{y}) $ is a maximum margin classifier of the form,
	\begin{flalign*}
	& \underset{\bm{\alpha}}{\arg\max} \, \sum_{i=1}^l \alpha_i -\frac{1}{2} \bm{\alpha}^T \bm{Y J} \hat{\bm{K}}_\eta ( \bm{I} + 2\beta \hat{\mathcal{L}}_{\eta} \hat{\bm{K}}_\eta )^{-1} \bm{J}^T \bm{Y \alpha} \\
	& \quad + \sum_{i=1}^l \log(1-\frac{\alpha_i}{C}) \quad \text{ s.t. } \sum_{i=1}^l y_i \alpha_i = 0, \, \alpha_1, \dots, \alpha_l \ge 0 \\
	& \underset{b}{\arg\min} \hspace{-12pt} \sum_{s \in \{i | \hat{\alpha}_i \neq 0 \} } \hspace{-12pt} |(y_s - k(\bm{X}_{s}, \bm{X}_{\hat{\eta}}) ( \bm{I} + 2\beta \hat{\mathcal{L}}_{\eta} \hat{\bm{K}}_\eta)^{-1} \bm{J}^T \bm{Y} \hat{\bm{\alpha}}) - b| 
	\end{flalign*}
	where $ \hat{\Omega}_\phi $ is approximated with the GEM principle described in subsection \ref{GEM} and $ \hat{\bm{K}}_\eta $ is $ a \times a $ submatrix of only $ \{ \bm{X}_i : \hat{\eta}_i = 1 \} $.
\end{theorem}

LatLapMED exploits the idea that all high utility points are anomalous because the similarity metric in $ \hat{\Omega}_\phi $ is dependent on the decision boundary and label information; thus it will be skewed away from points with high utility neighbors. This is crucial because any high utility points incorrectly estimated as nominal will not be considered in the M-step and thus cannot be predicted as high utility. In contrast, it is not that vital to correctly estimate low utility anomalous points because it is not of interest to distinguish between them and the nominal. As the decision boundary moves every EM iteration, it changes the penalties that neighboring nodes can incur in the similarity metric. Since the normalized margin is 1, setting $\rho=1$ is typical; however, if the data is difficult to classify, it may be appropriate to set $\rho > 1$ because there is less confidence in the classification. Thus the threshold $ \rho $ can be set empirically using prior domain knowledge of the structure of the data or by cross-validation. 

\begin{corollary} \label{decision_rule}
Once the EM algorithm converges, the decision rule is 
\begin{flalign*} 
& \hat{y}_{i'} = 
\begin{cases}
 -1 \qquad \emph{ if } \hat{\eta}_{i'} = 0 &\\
\emph{sign} \left( k(\bm{X}_{i'}, \bm{X}_{\hat{\eta}}) ( \bm{I} + 2\beta \hat{\mathcal{L}}_{\eta} \hat{\bm{K}}_\eta)^{-1} \bm{J}^T \bm{Y} \hat{\bm{\alpha}} + \hat{b} \right) \,\, \emph{o.w.} &
\end{cases} 
\end{flalign*}
where $ \bm{X}_{\hat{\eta}} $ are only the data points estimated to be anomalous and $ \hat{\bm{\alpha}}, \hat{b} $ are final optimal parameters.
\end{corollary}

In this work, we approximate $ \E_{\bm{\eta}}(\bm{\eta} \bm{\eta}^T | \bm{X}, \bm{\alpha}^{t-1}) $ using the GEM principle with similarity metric \eqref{sim_metric} because the expectation distribution is unknown. Because the GEM principle is nonparametric, it does not impose, potentially incorrect, distributional assumptions on the unknown distribution of anomalies, which may be extremely difficult to parametrically characterize. Other estimators, derived using either a different similarity metric in the GEM principle or another nonparametric method altogether, could be used instead in the E-step. We believe our estimator is a good choice because it is asymptotically consistent and empirically we find it is sufficient enough such that the objective in the M-step increases every iteration. The LatLapMED algorithm, summarized below, produces a joint estimate of both anomaly and utility labels. This simultaneous estimation allows the method to incorporate additional information that would be lost when estimating the anomaly and utility labels independently. 

\setlength{\extrarowheight}{2pt}
\setlength{\arrayrulewidth}{1pt}
\begin{table}[h]
	\centering
	\begin{tabular}{l} \hline 
		\textbf{Algorithm 1:} LatLapMED \\ \hline 
		\begin{minipage}{.9\linewidth}
			\smallskip
	\begin{algorithmic} 
		\STATE {\bfseries Input: } $ \phi, \rho, k, C, \beta, \bm{X}, \bm{y} $
		\REPEAT
		\STATE E-Step:
		\STATE \quad 1) Given $\hat{\bm{d}}^{t-1} = \hat{\bm{K}}_\eta ( \bm{I} + 2\beta \hat{\mathcal{L}}_{\eta} \hat{\bm{K}}_\eta)^{-1} \bm{J}^T \hspace{-3pt} \bm{Y} \hat{\bm{\alpha}}^{t-1} \hspace{-1pt} + \hat{b}^{t-1} $ 
		\STATE \quad 2) $ \hat{\eta}_i = \mathcal{I}(\bm{X}_i \notin \hat{\Omega}_\phi) $ where $ \hat{\Omega}_\phi $ is the solution of \eqref{accept_region} \\
		\STATE M-Step: 
		\STATE \quad 1) Given $\hat{\bm{\eta}}$, form new submatrices $ \hat{\bm{K}}_\eta $ and $\hat{\mathcal{L}}_{\eta} $
		\STATE \quad 2) Solve the objectives in Theorem \ref{EM_alg} to get $\hat{\bm{\alpha}}^t, \hat{b}^t $
		\UNTIL{convergence}
		\STATE {\bfseries Return: } $ \hat{\bm{\eta}}, \hat{\bm{\alpha}}, \hat{b} $
	\end{algorithmic}
	\label{myalgo}
\smallskip
\end{minipage} \\ \hline 
\end{tabular}
\end{table}

\subsubsection{Computational Complexity}

The E-step uses a K-kNNG for the estimators $ \hat{\eta}_i $. This K-kNNG is defined by the Euclidean distance between points, which is constant over all EM iterations, and a penalty, which changes between EM iterations. The Euclidean distances are calculated and the $k$ neighbors are sorted only once at initialization, which have computational complexity $O(n^2 p + n^2 \log(n))$. At every EM iteration, the E-step just needs to add the $n$ penalties given from the previous M-step and sort the $n$ total edge lengths, adding computational complexity $O(n \log (n))$ per iteration. The M-step maximizes a quadratic objective (formed from $O(a^3)$ matrix operations) over the Lagrange multipliers $\alpha_i$, which can be solved with conic interior point methods in polynomial $a$ time, where $a << n$. Alternatively, the M-step objective can be approximated as a quadratic program and solved using sequential minimal optimization in linear $a$ time. Thus, the LatLapMED algorithm has overall computation time $O(n^2 p + n^2 \log(n) + \text{\#iter}*(n \log(n)+a^3+a^q))$ where $1 <= q << \infty$ depends on how the objective is solved. If we assume that the computational time of the E-step dominates significantly over the computational time of the M-step because $a << n$, then this reduces to roughly $O(n^2 p + n^2 \log(n) + \text{\#iter}*(n \log (n)))$ . We have had no problem implementing the LatLapMED algorithm even for $n$ as high as 100,000 points. Parallelization of the initial sorted distances for the K-kNNG can also improve its computational speed to $O(\frac{n^2 p + n^2 \log(n)}{\# nodes})$.

The final LatLapMED posterior $ \P(\bm{\theta}, b, \bm{\gamma}, \lambda | \bm{X}, \hat{\bm{\eta}}, \bm{y}) $, where $ \hat{\bm{\eta}} $ is the estimated latent variables at EM convergence, is a probabilistic model with a mode that performs maximum margin classification. Thus LatLapMED has the classification robustness of discriminant methods, but the natural flexibility of generative methods to incorporate latent variables. Additionally the generative nature also provides for sequential classification by using the posterior distribution as a new prior for new data in the MED framework. This allows LatLapMED to be very applicable to real world problems where data is often continuously collected in a stream. Alternatively, it can also be used to process a very large dataset, $n >> 10^5$, in smaller batches allowing for the LatLapMED algorithm to be feasible for very large $n$.

\subsubsection{Limitations and Future Work}

While the computational complexity of the LatLapMED algorithm is feasible for moderately large datasets, it is still more computationally expensive than many competing methods. However the performance improvement may make it worthwhile to implement the proposed algorithm in challenging anomaly detection problems. Strategies for reducing computational complexity through parallelization, specialized hardware approaches, or implementation of second order acceleration methods are also possible. Additionally, the problem of online sequential anomaly detection and classification is open. One possible approach would be to make the E-step be only weakly dependent on of the prior information to make it adaptive to changes in the prior over time. An alternative solution would be to modify the K-kNNG in the GEM algorithm to incorporate a time varying prior through weighted edges or a suitable choice of level set boundary that varies with the prior. Finally the number of tuning parameters in the LatLapMED algorithm might be reduced by using hyperpriors or empirical risk minimization methods.

\section{Experiments} 

In this section, we apply the LatLapMED algorithm to both simulated and real data sets and demonstrate that the proposed method outperforms alternative two-stage methods that first estimate the anomaly labels and then predict the utility labels of only the estimated anomalous points. For combination in the two-stage methods, we consider three algorithms for non-parametric anomaly detection and both popular supervised and semi-supervised algorithms for classification; these are shown in Table \ref{2stage_algs}.

\renewcommand{\arraystretch}{1.3}
\begin{table}[h]
	\caption{Algorithms Used to Form Two-Stage Methods} 
	\label{2stage_algs}
	\centering
	\begin{tabular}{ccccc}
		\hline
	 Anomaly Detector & & Supervised Classifier & & Semi-Supervised \\
		\hline
		GEM & & SVM & & LapMED \\
		1SVM & + & RF & or & LapSVM \\
		SSAD & & NN & & LDS \\
		\hline
	\end{tabular}
\end{table}

The one class SVM (1SVM) of \cite{NIPS1999_1723} and the standard GEM with euclidean distance K-kNNG of \cite{NIPS2006_3145} are unsupervised, but the semi-supervised anomaly detection (SSAD) algorithm of \cite{gornitz2013toward} incorporates the labeled points as known anomalous points. In the three supervised methods: SVM, random forests (RF) of \cite{breiman2001random}, and neural networks (NN), we train the algorithms with labeled points and predict the labels of only the anomalous unlabeled points, and in the three semi-supervised methods: the LapMED from Section \ref{lapMED}, the LapSVM of \cite{Belkin:2006:MRG:1248547.1248632}, and the low density separation (LDS) algorithm of \cite{chapelle2005semi}, we train the algorithms on all anomalous points to classify their unlabeled ones. Because these two-stage methods naively perform anomaly detection independently of classification, there is no synergy between the two stages unlike in the LatLapMED method, which binds the two actions through the EM algorithm.

For all of the following experiments, we choose the parameters of classifiers based on the methods described in their original papers. We verify that our parameter choices are acceptable because under ``oracle" conditions where the anomaly labels are known, all classifiers can classify relatively equally as well, which is to be expected. All methods are implemented in MATLAB, but most of the optimization is done with an optimization package written in another language. Specifically, we use LIBSVM \cite{CC01a} for the SVM classifiers, CVX \cite{gb08, cvx} for optimization of the LapMED objective, CVX or LIBQP \cite{libqp} for SSAD, the code provided in \cite{chapelle2005semi} for LDS, and the corresponding MATLAB toolboxes for random forest and neural networks. Thus the GEM routine in the E-step of LatLapMED is solved purely in MATLAB, but the LapMED objective in the M-step is solved with CVX. Because the high utility class is much smaller than the low utility class, we choose to use precision and recall to measure the performance of all methods due to the benefits argued in \cite{davis2006relationship}.

\subsection{Simulation Results} \label{sim}

We simulate datasets of sample size 7,000 where the variables come from a multivariate folded t-distribution with location $\bm{\mu} = \bm{0}$, a random positive definite scale matrix $\bm{\Sigma}$, and 30 degrees of freedom. We calculate the utility scores for each point by \\ $ score_i = \underset{h}{\max} \, \frac{1}{|\mathcal{C}_h|} \sum_{ j \in \mathcal{C}_h} \bm{X}_{ij} -\frac{1}{p - |\mathcal{C}_h|} \sum_{ j \notin \mathcal{C}_h} \bm{X}_{ij} $ where $\mathcal{C}_h$ is a random set of column indicies for random utility component $h$. Thus 5\% of the data is anomalous and the top 25\% of anomalies with the highest utility scores are defined as having high utility. We observed 30\% of the high utility anomalies and an equal number of low utility anomalies. 

In the exact simulations below, we use the parameters listed in Table \ref{alg_params}. For SSAD, we allow the regularization parameter for margin importance $\kappa$ to vary. For LatLapMED, we set $ \rho = 1 $ because we believe, in the space of only the anomalies, the data is pretty separable and easily classified. Figure \ref{box_rates}a shows the ``oracle" scenario, where anomaly labels are known; so the nominal points, $ \eta_i = 0 $, are automatically given a label $ \hat{y}_i = -1$, and the semi-supervised methods (LapMED, LapSVM, LDS) train and classify on only points with $ \eta_i = 1 $ while the supervised methods (SVM, RF, NN) predict on the unlabeled $ \eta_i = 1 $ points. Under this scenario, all classifiers have relatively equal precision and recall, which indicates that our parameter choices are acceptable since each classifier has its advantages and disadvantages. Note that LapMED and LapSVM are essentially the same model as discussed theoretically in Section \ref{lapMED}. Additionally note that when the anomaly labels are known, we have the complete posterior for LatLapMED described in Section \ref{sec:complete_post}, which has the same mode as the LapMED posterior given only anomalous data.

\renewcommand{\arraystretch}{1.3}
\begin{table}[h]
	\caption{Parameters Used in the Algorithms} 
	\label{alg_params}
	\centering
	\begin{tabular}{p{1.4cm}l}
		\hline 
		\makebox[\dimexpr(\width-2em)][l]{\textbf{Anomaly Detector}} & \multicolumn{1}{c}{\textbf{Parameters}} \\
		GEM & $k = 10$ neighbors in kNN graph, the $K$ points = $\phi n$ \\
		1SVM & $\sigma = 1$ in rbf kernel, $\nu = \phi$\\
		SSAD & $\sigma = 1$ in rbf kernel, $\kappa$, label $ C = 1$, unlabel $C = 1/\phi$\\
		\hline
		\textbf{Classifier} & \multicolumn{1}{c}{\textbf{Parameters}} \\
		SVM & $\sigma = 1$ in rbf kernel, cost $C = 50$ \\ 
		LapSVM & $\Uparrow$, $\beta = \frac{10 C l}{a^2} $, Laplacian: $k=50$, $\tau = 100$ in heat kernel \\
		LapMED & $\Uparrow$ (same as above in LapSVM) \\
		LDS & $k=50$ neigh., $\sigma = 1$ in rbf, $C = 50$, softening = 1.5 \\
		RF & 50 weak learners, default params. in MATLAB toolbox \\
		NN & 50 neurons, default params. in MATLAB toolbox \\
		\hline
		\makebox[\dimexpr(\width-2em)][l]{\textbf{Joint Method}} & \multicolumn{1}{c}{\textbf{Parameters}} \\
		LatLapMED & $\Uparrow$ (same as above in LapMED), threshold $ \rho = 1, \phi $ \\
		\hline
	\end{tabular}
\end{table}

\begin{figure*}[h]
	\centering
	\includegraphics[width=\textwidth]{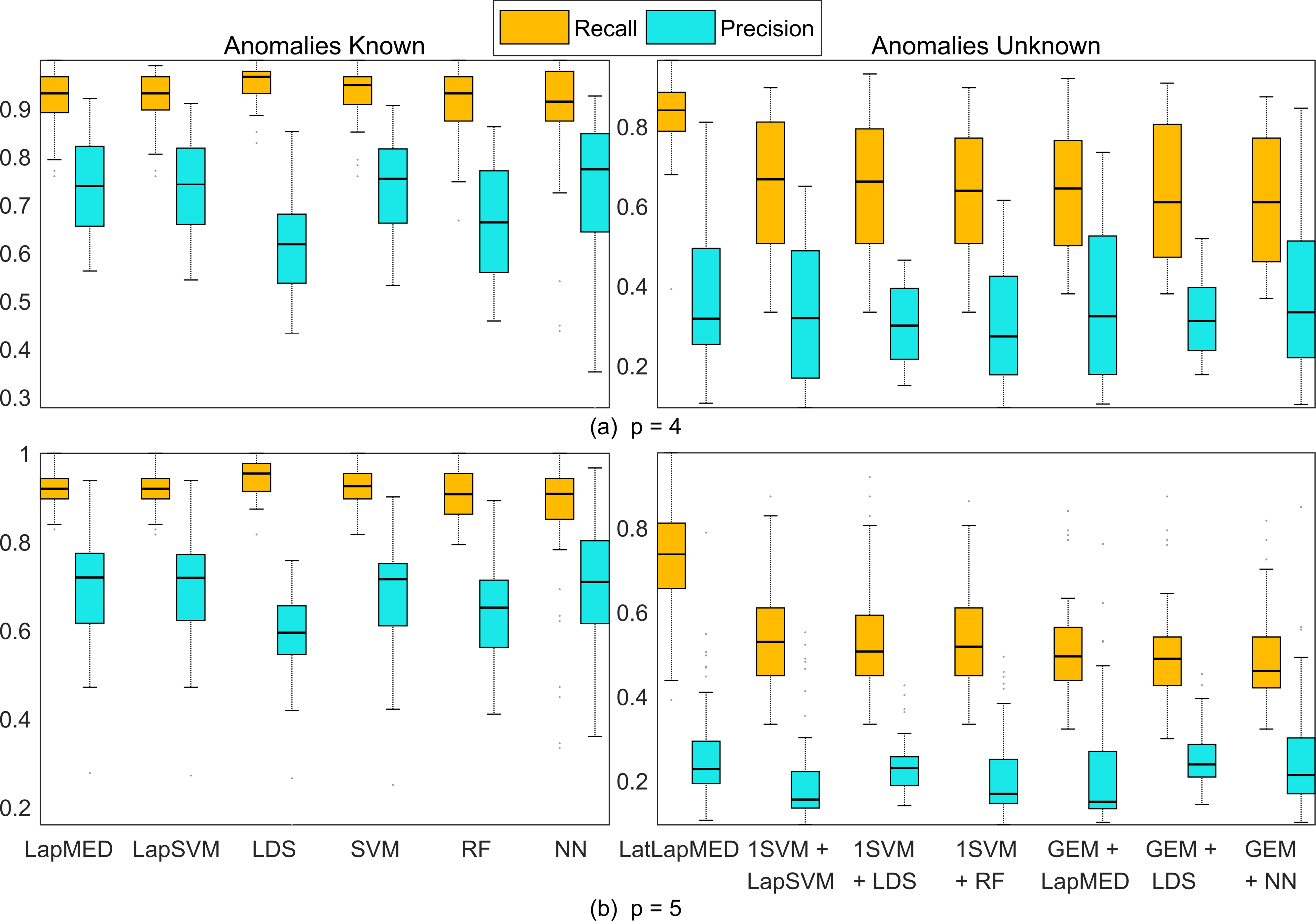} 
	\caption{Boxplots showing 50 trials of precision and recall of different methods. a) Under the ``oracle" scenario, where anomaly labels are known, all classifiers have relatively equal performance. b) The anomaly labels are unknown, but the percentage of the data that is anomalous is known to be $\phi=0.05$. All methods have relatively equal precision, but the LatLapMED method has much better recall because it does not treats the utility and anomaly labels as independent.} \label{box_rates}
\end{figure*} 

In a realistic scenario, as opposed to the ``oracle" one, the anomaly labels are unknown and must be estimated. So we compare our LatLapMED method, which estimates the anomaly and utility labels simultaneously, with two-stage methods that first perform either GEM or 1-class SVM for anomaly detection and then uses one of the above classifiers to label the utility of the anomalous points. In Figure \ref{box_rates}b, we show similar boxplot plots to the ones in Figure \ref{box_rates}a, but in this scenario, the anomaly labels are latent. While the LatLapMED method has similar precision as the alternative two-stage methods, it has much better recall. This indicates that LatLapMED is able leverage more information from the labeled anomalous points than a naive two-stage method that treats the utility label information and anomaly status of points as independent.

Figure \ref{PR_curves} compares LatLapMED against all combinations of alternative two-stage methods in $p = 3$ and $p=6$ dimensions respectively. The Precision-Recall (PR) curves (averaged over 50 trials) show that for all levels of $ \phi $ LatLapMED always dominates all of the naive two-stage methods. It is well known that as dimensionality increases, nonparametric estimation becomes more difficult, so the performance of all methods degrade because anomaly detection becomes more difficult. However, Table \ref{auc_table} shows that LatLapMED always has superior performance over the other methods irrespective of the dimension. 

\begin{figure}[h]
\centering
\subfloat[$p=3$]{\includegraphics[width=\columnwidth]{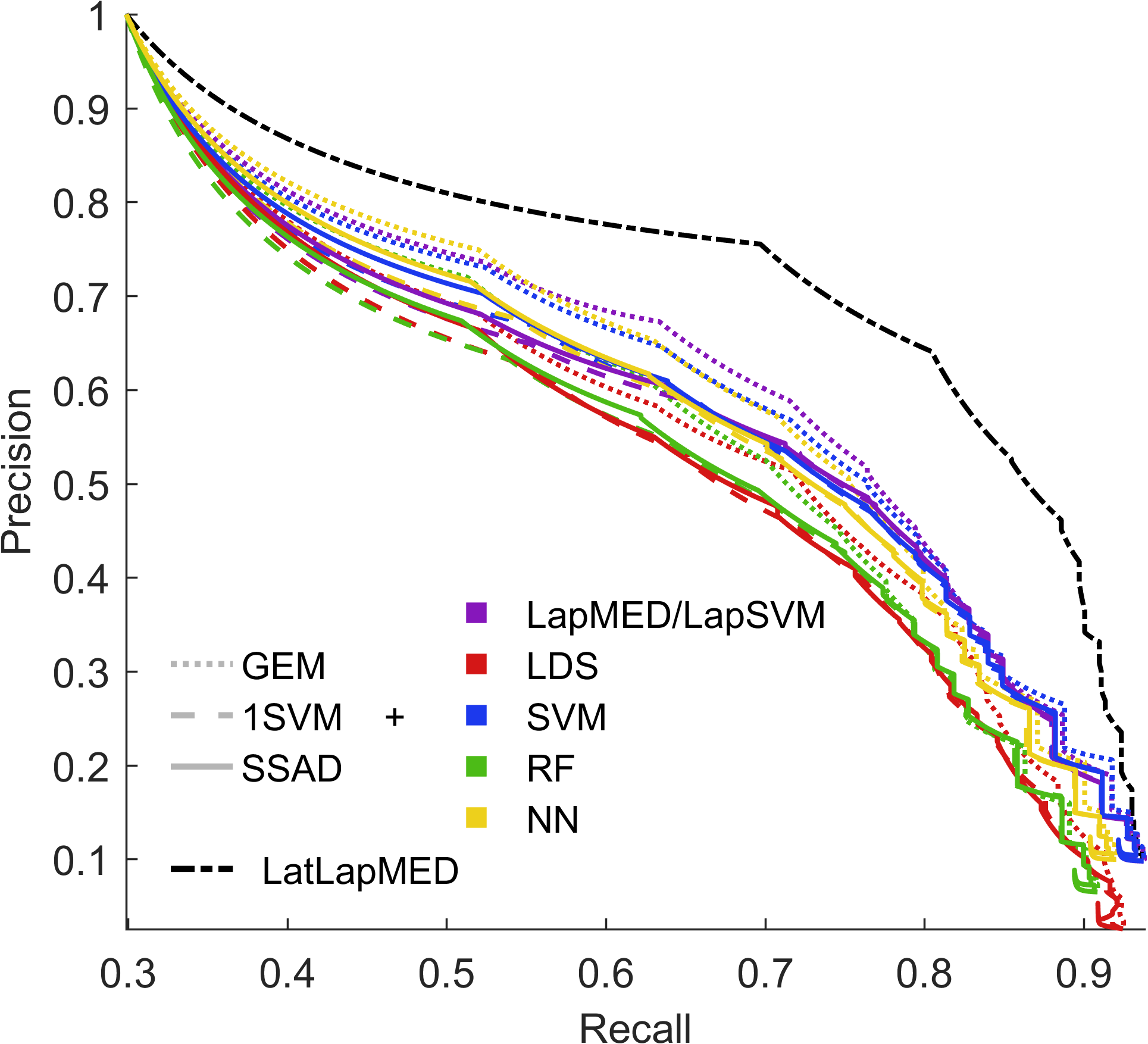}}
\hfil
\subfloat[$p=6$]{\includegraphics[width=\columnwidth]{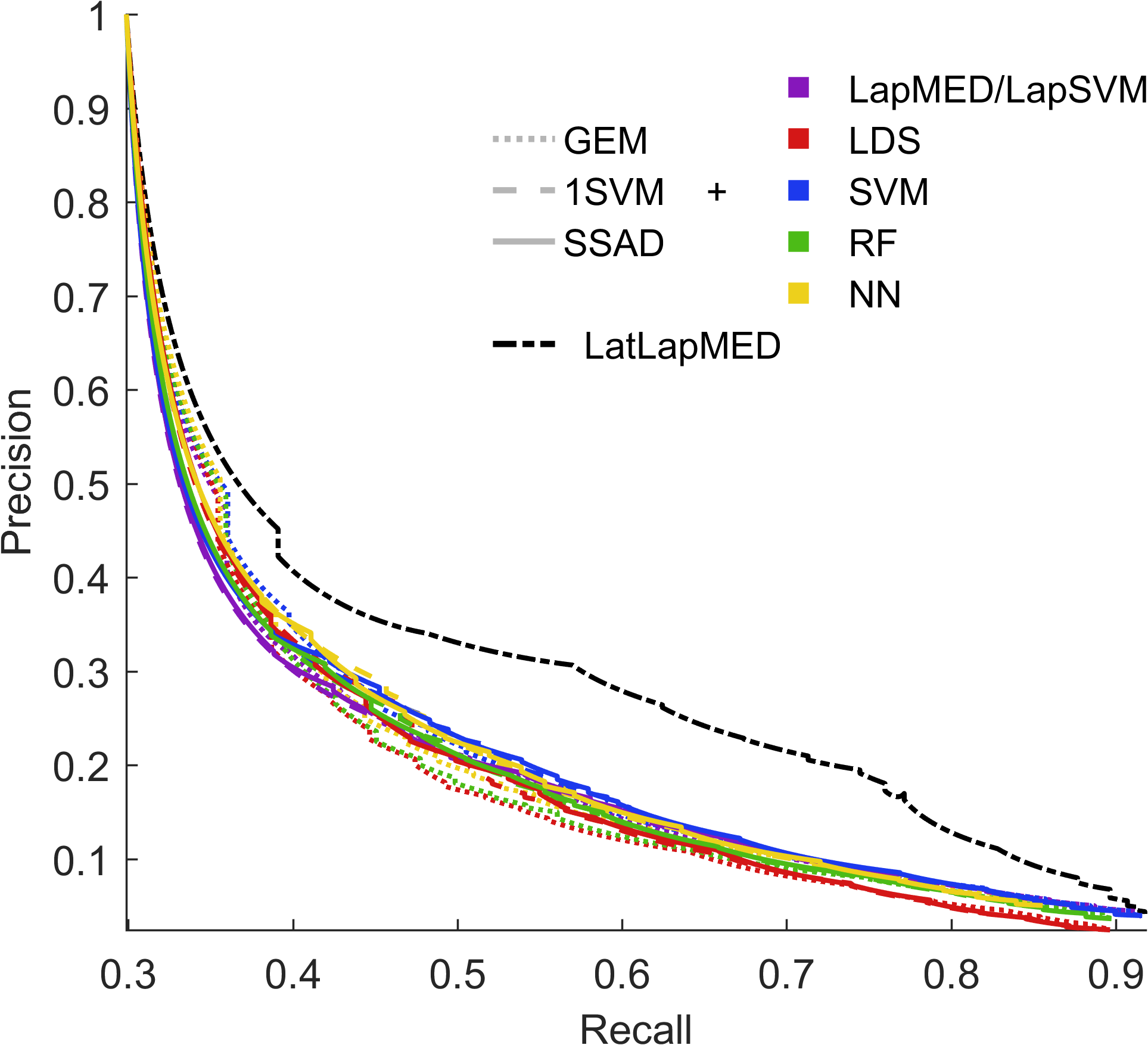}}
\caption{PR curves for various anomaly levels $\phi$ in 3 and 6 dimensions. The area under the PR curves are listed in Table \ref{auc_table}. The LatLapMED method significantly outperforms all the naive two-stage methods.}
\label{PR_curves}
\end{figure} 

\renewcommand{\arraystretch}{1.3}
\begin{table}[h]
\caption{Area Under the PR Curve (AUC-PR)} 
\label{auc_table}
\centering
\begin{tabular}{lcccc}
\hline
& $p = 3$ & $p = 4$ & $p = 5$ & $p = 6$ \\
\hline
GEM+LapMED &0.69246 &0.55129 &0.43539 &0.42204 \\
1SVM+LapSVM &0.66449 &0.54049 &0.439 &0.41696 \\
SSAD+LapSVM &0.66899 &0.54406 &0.44189 &0.41842 \\
GEM+LDS &0.65386 &0.53299 &0.44842 &0.41061 \\
1SVM+LDS &0.63228 &0.52483 &0.44607 &0.41522 \\
SSAD+LDS &0.63614 &0.52994 &0.45126 &0.41557 \\
GEM+SVM &0.68738 &0.56141 &0.43449 &0.42702 \\
1SVM+SVM &0.6675 &0.55206 &0.43766 &0.42286 \\
SSAD+SVM &0.67246 &0.55739 &0.44122 &0.42413 \\
GEM+RF &0.65483 &0.52194 &0.42837 &0.4151 \\
1SVM+RF &0.63192 &0.51556 &0.43516 &0.41667 \\
SSAD+RF &0.63823 &0.52169 &0.43913 &0.41737 \\
GEM+NN &0.68344 &0.54716 &0.44525 &0.41907 \\
1SVM+NN &0.66064 &0.54109 &0.44639 &0.42216 \\
SSAD+NN &0.66673 &0.54613 &0.45027 &0.42216 \\
LatLapMED &0.76253 &0.66417 &0.51792 &0.47854 \\
\hline
\end{tabular}
\end{table}

Figure \ref{trials} gives an in-depth view of LatLapMED compared to some alternative two-stage methods. The anomaly level $\phi$ of the methods is set to be between $0.05$ and $0.06$ to control the number of false positives. The number of false negatives in LatLapMED is much lower than that of the other methods. This is because unlike the two-stage methods, if LatLapMED misses some high-utility points when estimating anomalies, it can correct for them in the next EM iteration. Figure \ref{trial20} shows how both the number of false positives and false negatives decrease as the EM algorithm in LatLapMED iterates. In comparison to the naive two-stage GEM+LapMED, which would be equivalent to LatLapMED with only one EM iteration, LatLapMED is able to recover over 50\% of the high utility points initially missed in the first EM iteration.

\begin{figure}[h]
	\centering
		\includegraphics[width=\columnwidth]{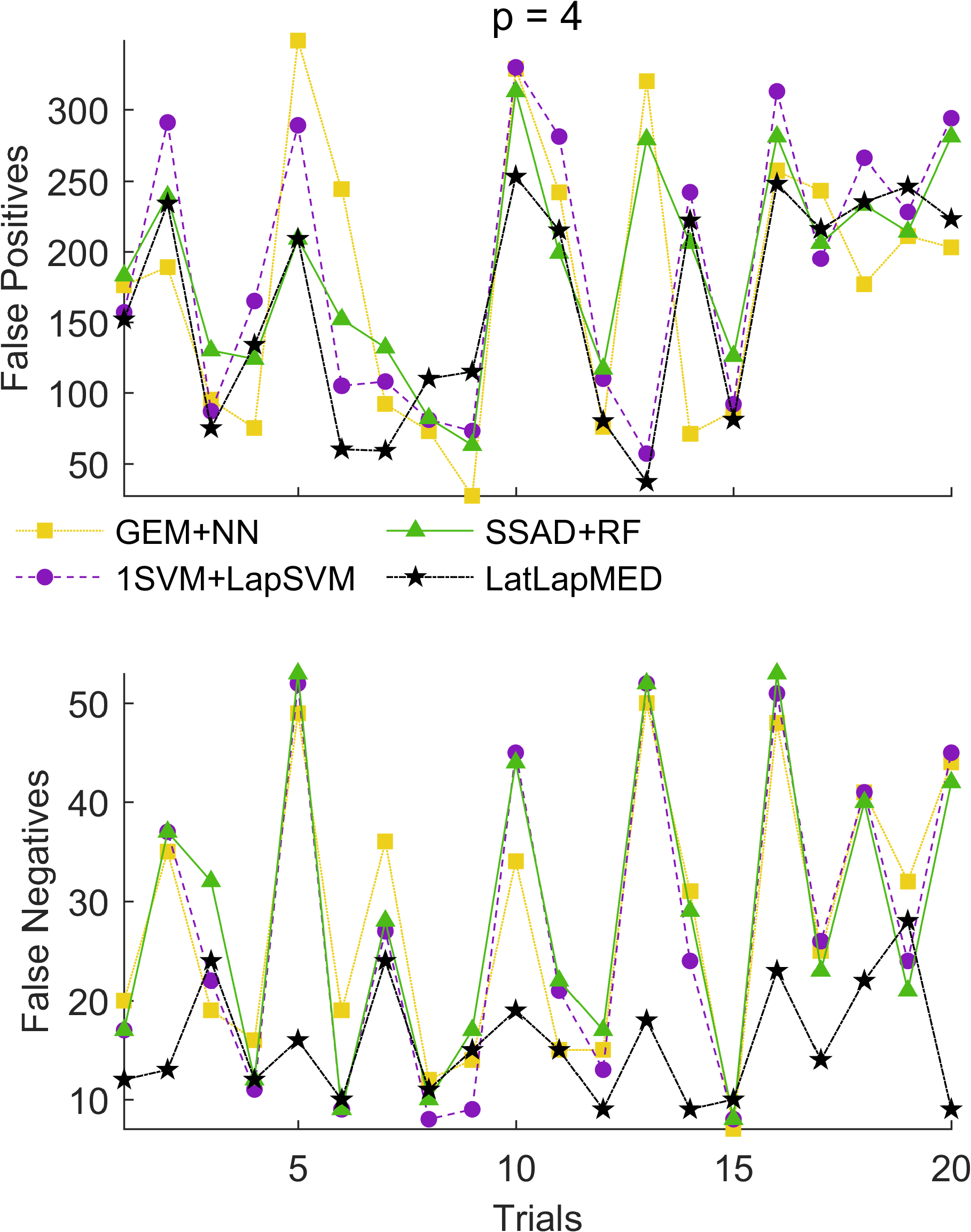}
		\caption{The number of false positives and false negatives in 20 different trials with $ \phi \in [0.05, 0.06]$ to control the number of false positives. LatLapMED has far fewer false negatives for the same number of false positives compared to the other methods.} \label{trials}
\end{figure} 

\begin{figure}[h]
\centering
\includegraphics[width=\columnwidth]{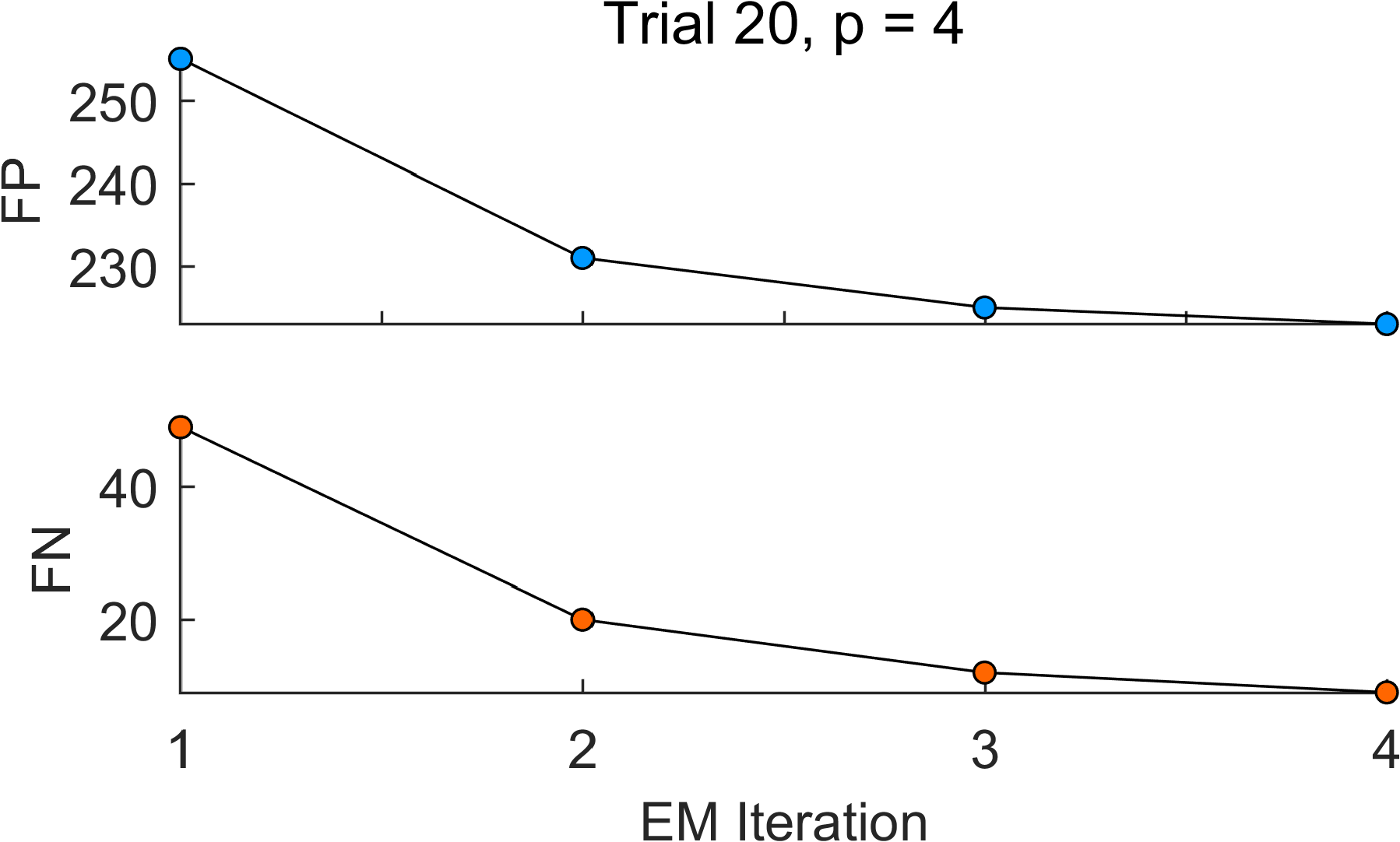}
\caption{The number of false positives (FP) and false negatives (FN) predicted by LatLapMED decrease as the EM iterations in the algorithm increase. This is due to the synergy between the anomaly detection in the E-step and the classification in the M-step.}
\label{trial20}
\end{figure} 

\renewcommand{\arraystretch}{1.3}
\begin{table}[h]
	\caption{Mean and standard deviation of CPU times over 50 trials.} 
	\label{speed}
	\centering
	\begin{tabular}{lcccc}
		\hline
		& Average CPU time & Standard Deviation \\
		\hline
SSAD+LapSVM &3.9784 &0.11991 \\
SSAD+LDS &4.2316 &0.18557 \\
SSAD+SVM &3.9491 &0.29182 \\
SSAD+RF &4.2800 &0.23246 \\
SSAD+NN &4.2534 &0.26558 \\
GEM+LapMED &1.4428 &0.13992 \\
GEM+LDS &1.7278 &0.19305 \\
GEM+SVM &1.0253 &0.05716 \\
GEM+RF &1.3775 &0.09409 \\
GEM+NN &1.2916 &0.11141 \\
1SVM+LapSVM &0.3750 &0.09295 \\
1SVM+LDS &0.7897 &0.19116 \\
1SVM+SVM &0.2106 &0.06824 \\
1SVM+RF &0.5984 &0.17712 \\
1SVM+NN &0.4781 &0.11465 \\
LatLapMED &2.9944 &0.64291 \\
		\hline
	\end{tabular}
\end{table}

In Table \ref{speed}, we show the mean and standard deviation of the CPU time in seconds for each algorithm over 50 trials. The algorithms were run on a quad-core Intel i7-6700HQ CPU at 3.20GHz using Matlab. While we have not numerically optimized each algorithm, we used as many built-in functions and optimizers, which are written in compiled languages (C++, Fortran), to show the best performance. LatLapMED is slower than many of the two-stage methods, but it is not exorbitantly slower, and it is still faster than two-stage methods that use SSAD.

\subsection{Experiment on Reddit data}

We apply LatLapMED to the May 2015 comments of the Reddit comment dataset \cite{Reddit}. We form a sample of subreddits with variables: \emph{Avg. Number of Users}, \emph{Avg. Gilded}, and \emph{Avg. Score}, where only subreddits with at least 100 comments are included and additionally only the top 7,000 most controversial subreddits are chosen (from approximately 10,000). The anomalous data points are defined as those that lie in the tail 3\% of any variable's marginal distribution and we are interested in only the controversial subreddits among these anomalous points. Thus, we treat the average controversy of each subreddit as a utility score, with again 30\% visible and the top 25\% as high utility. This mimics the situation where a domain expert is given roughly 1.5\% of the dataset that is considered to be anomalous, and asked to label it. 

Here the cost regularization parameter $C = 2$ is chosen to be smaller than in the simulations because we expect the margin to be noisier, and similarly the softening parameter in LDS is increased to $100$. The other parameters, which basically describe the structure of the classifiers, are the same as in the simulations. We choose $\phi$ to control the false positive rate (FPR) to be around $0.05$, which corresponds to a commonly chosen Type 1 error level. Table \ref{rates_table} shows the rates of LatLapMED, all the competing naive two-stage methods, and an ``oracle'' LapMED, which we use as a lower/upper bound on the best LatLapMED could do. For a Type 1 error level of $0.05$, LatLapMED does considerably better than the competing methods. It has the lowest false negative rates (FNR) and the highest recall. While 1SVM+RF and GEM+NN have slightly higher precision than LatLapMED, they also have much lower recall.

\renewcommand{\arraystretch}{1.3}
\begin{table}[h]
\caption{False Positive Rate, False Negative Rate, Recall, Precision $\,$ for Reddit Data} 
\label{rates_table}
\centering
\begin{tabular}{lcccc}
\hline
& FPR & FNR & Recall & Precision \\
\hline
{``oracle'' LapMED} & 0.02224 & 0.13008 & 0.86992 & 0.41797 \\
GEM+LapMED &0.050022 &0.17886 &0.82114 &0.22697 \\
1SVM+LapSVM &0.050167 &0.1626 &0.8374 &0.22991 \\
SSAD+LapSVM &0.050167 &0.17886 &0.82114 &0.22646 \\
GEM+LDS &0.050894 &0.19512 &0.80488 &0.22049 \\
1SVM+LDS &0.048422 &0.17073 &0.82927 &0.23448 \\
SSAD+LDS &0.049004 &0.1626 &0.8374 &0.23409 \\
GEM+SVM &0.056275 &0.15447 &0.84553 &0.21181 \\
1SVM+SVM &0.049295 &0.1626 &0.8374 &0.23303 \\
SSAD+SVM &0.048131 &0.17886 &0.82114 &0.2338 \\
GEM+RF &0.050749 &0.17073 &0.82927 &0.22616 \\
1SVM+RF &0.047114 &0.17886 &0.82114 &0.23765 \\
SSAD+RF &0.049731 &0.1626 &0.8374 &0.23146 \\
GEM+NN &0.047259 &0.17886 &0.82114 &0.23709 \\
1SVM+NN &0.047404 &0.18699 &0.81301 &0.23474 \\
SSAD+NN &0.047695 &0.19512 &0.80488 &0.23185 \\
LatLapMED &0.049149 &0.14634 &0.85366 &0.23702 \\
\hline
\end{tabular}
\end{table}

Additionally compared to ``oracle" LapMED, LatLapMED does not do considerably worse. Its precision is not nearly as high as the ``oracle" method's; however, 191 out of the 338 subreddits incorrectly predicted to be controversial (false positives), actually have controversy scores in the top 25\%, but since they are not anomalous, they are not labeled as high utility by our criteria. This is very promising because it implies that our method is able to additionally find high utility points that may not lie far enough in the tails of the empirical distribution. The recall of LatLapMED is almost as high as that of the ``oracle" method's with LatLapMED only failing to label as controversial (false negatives) the subreddits { \small [`vegetarian', `DesignPorn'] } compared to the ``oracle". Otherwise, both methods failed to find the other 16 subreddits: \\
{ \small [`pathofexile', `Cleveland', `Liberal', `mississauga', `Eesti', `Images', `uofmn', `trackertalk', `Kuwait', `asianbros', `saskatchewan', `rule34\_comics', `boop', `macedonia', `wanttobelieve', `DebateACatholic'] }. \\
While some of these topics are definitely controversial, others such as `mississauga' and `saskatchewan' (providences of Canada) or `uofmn' (University of Minnesota) seem to have unreasonably high controversy scores. It is not particularly worrisome that LatLapMED failed to predict these topics as controversial because the ``oracle'' also incorrectly classified them, so many of them could be considered mislabeled by the domain expert. 

\subsection{Experiment on CTU-13 data}

Finally, we apply LatLapMED to the CTU-13 dataset, which is of botnet traffic on a university network that was captured by CTU University, Czech Republic, in 2011 \cite{CTU-13}. The dataset contains real botnet traffic mixed with normal traffic and background traffic. The authors of \cite{CTU-13} processed the captured traffic into bidirectional NetFlows and manually labeled them. The dataset contains 13 different scenarios and for our experiments below we considered two scenarios, 1 and 8. Scenario 1 contains the malware Neris.exe, which is a bot that sent spam, connected to an HTTP CC, and used HTTP to do ClickFraud. Scenario 8 has malware QvodSetuPuls23.exe, which contacted many different Chinese C\&C hosts, received large amounts of encrypted data, and scanned and cracked the passwords of machines. We are interested in identifying the botnet traffic (high-utility points) from the rare, but uninteresting normal traffic (low-utility, anomalous points) and the background traffic (nominal points) in a situation where instead of manually labeling all points, only a small subset is labeled.

For each scenario, we randomly sample 38,000 NetFlows of background traffic and 1000 NetFlows each of normal and botnet traffic, making 5\% of the samples anomalous. We allow 300 of the normal and 300 of the botnet traffic to have visible labels so a domain expert would only be manually labeling 1.5\% of all the samples in the dataset. We used 9 of the features provided by the NetFlows dataset: duration of the flow, direction of the flow, total packets, total bytes, source bytes, source and destination port numbers and IP addresses (in integer format). Thus each of the two datasets have dimensions $p$ = 9 features and $n$ = 40,000 total samples, of which $a$ = 2,000 are anomalous and $l$ = 600 are labeled. In order to have multiple trials, we perform this sampling 10 times so that we have 10 almost independent experiments for each scenario. The following results are the average of these 10 trials.

Because many of the features are discrete and not continuous, we use cosine distances and cosine kernels instead of euclidean distances and the radial basis kernel; otherwise, the parameters are the same as in Table \ref{alg_params}. We choose $\phi$ to so that the Type 1 error level (or FPR) is $0.01$. Like in the Reddit experiments, we compare LatLapMED against all the competing naive two-stage methods and an ``oracle'' LapMED and summarize the performance in Tables \ref{s1_rates_table} and \ref{s8_rates_table} for scenarios 1 and 8 respectively. The only two-stage method we do not compare against are those using SSAD due to its unmanageably high computational complexity.

\renewcommand{\arraystretch}{1.3}
\begin{table}[h]
	\caption{Mean False Pos. Rate, False Neg. Rate, Recall, Precision \hspace{\textwidth} for Scenario 1 over 10 trials} 
	\label{s1_rates_table}
	\centering
	\begin{tabular}{lcccc}
		\hline
		& FPR & FNR & Recall & Precision \\
		\hline
``oracle' LapMED &0.0038308 &0.3457 &0.6543 &0.82145 \\
GEM+LapMED &0.012279 &0.6916 &0.3084 &0.39202 \\
GEM+LDS &0.010464 &0.6929 &0.3071 &0.4317 \\
GEM+SVM &0.010049 &0.6936 &0.3064 &0.44034 \\
GEM+RF &0.0062231 &0.6924 &0.3076 &0.56775 \\
GEM+NN &0.010844 &0.691 &0.309 &0.43359 \\
1SVM+LapSVM &0.012713 &0.6836 &0.3164 &0.38969 \\
1SVM+LDS &0.0098 &0.6816 &0.3184 &0.45581 \\
1SVM+SVM &0.0093462 &0.6853 &0.3147 &0.4641 \\
1SVM+RF &0.0078769 &0.6742 &0.3258 &0.51662 \\
1SVM+NN &0.0089462 &0.6762 &0.3238 &0.48675 \\
LatLapMED &0.0094692 &0.402 &0.598 &0.61899 \\
		\hline
	\end{tabular}
\end{table}

\renewcommand{\arraystretch}{1.3}
\begin{table}[h]
	\caption{Mean False Pos. Rate, False Neg. Rate, Recall, Precision \hspace{\textwidth} for Scenario 8 over 10 trials} 
	\label{s8_rates_table}
	\centering
	\begin{tabular}{lcccc}
		\hline
		& FPR & FNR & Recall & Precision \\
		\hline
``oracle' LapMED &0.0031538 &0.15 &0.85 &0.87359 \\
GEM+LapMED &0.011538 &0.6979 &0.3021 &0.41185 \\
GEM+LDS &0.011118 &0.6997 &0.3003 &0.42204 \\
GEM+SVM &0.011267 &0.6969 &0.3031 &0.42978 \\
GEM+RF &0.007859 &0.6967 &0.3033 &0.49887 \\
GEM+NN &0.0095487 &0.696 &0.304 &0.4506 \\
1SVM+LapSVM &0.0092205 &0.6916 &0.3084 &0.4729 \\
1SVM+LDS &0.010682 &0.6919 &0.3081 &0.43073 \\
1SVM+SVM &0.011746 &0.6898 &0.3102 &0.43154 \\
1SVM+RF &0.010272 &0.6807 &0.3193 &0.44506 \\
1SVM+NN &0.0096103 &0.6818 &0.3182 &0.46041 \\
LatLapMED &0.011064 &0.1779 &0.8221 &0.65125 \\
		\hline
	\end{tabular}
\end{table}

The average error rates of the ``oracle" method shown in Tables \ref{s1_rates_table} and \ref{s8_rates_table} indicates that identifying the botnet traffic is not extremely difficult when the anomalies are known. However, when the anomaly indicator variables are latent or unknown, the tables show that the problem is more difficult. Nonetheless, in both scenarios, LatLapMED has the lowest false negative rates (FNR) and the highest precision and recall. The most competitive two-stage methods do not come close to the performance of LatLapMED, and particularly in scenario 8, LatLapMED has significantly higher precision and recall. This is a direct result of the fact that all malware are statistical outliers, so incorporating label information into anomaly detection helps to identify botnet traffic.

\renewcommand{\arraystretch}{1.3}
\begin{table}[h]
	\caption{Mean and standard deviation of CPU times (in seconds) \hspace{\textwidth} for Scenario 1 over 10 trials} 
	\label{s1_CPU}
	\centering
	\begin{tabular}{lcccc}
		\hline
		& Average CPU time & Standard Deviation \\
		\hline
GEM+LapMED &31.2625 &0.58902 \\
GEM+LDS &27.7656 &0.68016 \\
GEM+SVM &26.55 &0.68274 \\
GEM+RF &26.9656 &0.82544 \\
GEM+NN &27.0109 &1.2891 \\
1SVM+LapSVM &4.1141 &0.10951 \\
1SVM+LDS &3.3328 &0.22134 \\
1SVM+SVM &1.9734 &0.08137 \\
1SVM+RF &2.6484 &0.18001 \\
1SVM+NN &2.5859 &0.57919 \\
LatLapMED &56.2547 &15.4514 \\
		\hline
	\end{tabular}
\end{table}

\renewcommand{\arraystretch}{1.3}
\begin{table}[h]
	\caption{Mean and standard deviation of CPU times (in seconds) \hspace{\textwidth} for Scenario 8 over 10 trials} 
	\label{s8_CPU}
	\centering
	\begin{tabular}{lcccc}
		\hline
		& Average CPU time & Standard Deviation \\
		\hline
GEM+LapMED &31.7359 &0.66264 \\
GEM+LDS &29.4672 &0.28389 \\
GEM+SVM &26.4953 &0.24859 \\
GEM+RF &26.9734 &0.31724 \\
GEM+NN &26.9922 &0.6696 \\
1SVM+LapSVM &4.3141 &0.0814 \\
1SVM+LDS &6.8063 &0.22136 \\
1SVM+SVM &1.9953 &0.10154 \\
1SVM+RF &2.6875 &0.10725 \\
1SVM+NN &2.6219 &0.55448 \\
LatLapMED &50.3656 &0.84805 \\
		\hline
	\end{tabular}
\end{table}

We also measure the CPU times of the two scenarios using the same Intel CPU and code as described in the simulations of subsection \ref{sim}. Tables \ref{s1_CPU} and \ref{s8_CPU} show that while LatLapMED is significantly slower than all the competing methods, it on average takes less than 1 minute to process on dataset of 40,000 NetFlows, which is still very reasonable.

\section{Conclusion}

We have proposed a novel data-driven method called latent Laplacian minimum entropy discrimination (LatLapMED) for detecting anomalous points that are of high utility. LatLapMED extends the MED framework to simultaneously handle semi-supervised utility labels and incorporate anomaly information estimated by off-the-shelf anomaly detection methods via EM. Through this extended framework, LatLapMED exploits the key idea that high-utility points are also anomalous. This allows the method to work successfully when provided with a very small number of utility labels. Our simulation results show its advantages over combinations of standard anomaly detection and classification algorithms. In particular, theses two-stage approaches perform worse because they treat statistical rarity and label information as independent components, which LatLapMED overcomes by explicitly combining them through a latent variable model and the EM algorithm. This performance increase is shown in the EM iterations of LatLapMED where using previous label information helps identify anomalies and vice versa. Finally, we applied our method to the Reddit and CTU-13 botnet datasets to show its applicability in real life situations where only certain high-utility anomalies are of interest to the end user.

\appendices

\section{Proofs for Section III}

\begin{proof}[Proof of Proposition \ref{lapmed}]
The posterior $ \P(\bm{\theta}, b, \bm{\gamma}, \lambda | \bm{X}, \bm{y}) $ is a log concave distribution where the log posterior can be treated as a Lagrangian function. So the MAP estimator $ \hat{\bm{\theta}} $ is the solution to $ \frac{\partial}{\partial \bm{\theta}} \log \left( \P(\bm{\theta}, b, \bm{\gamma}, \lambda | \bm{X}, \bm{y}) \right) = \sum_{i=1}^l \alpha_i y_i \bm{X}^T_i - ( \bm{I} + 2 \beta \bm{X}^T \mathcal{L}\bm{X}) \bm{\theta} = \bm{0} $
and the Lagrange multipliers $ \bm{\alpha} $ are the solution to
\begin{flalign*} 
& \frac{\partial}{\partial \bm{\alpha}} \log \left( \P(\bm{\theta}, b, \bm{\gamma}, \lambda | \bm{X}, \bm{y}) \right) |_{\bm{\theta} = \hat{\bm{\theta}}} = \frac{\partial}{\partial \bm{\alpha}} -\log \left( Z(\bm{\alpha}) \right) = \bm{0} \\
\end{flalign*}
where $ -\log \left( Z(\bm{\alpha}) \right) = $ Bias + Smoothness + $ \sum_{i=1}^l \text{Margin}_i $ + Weight, and integrating out each of the terms
\begin{flalign*} 
& \text{Bias:} -\log \left(\int_{-\infty}^\infty \frac{e^{-b^2 / 2 \sigma^2 }}{2 \pi \sigma^2} \exp \left\{ \sum_{i=1}^l \alpha_i y_i b \right\} db \right) \\
& = -\frac{\sigma^2}{2} \left(\sum_{i=1}^l \alpha_i y_i \right)^2 \Rightarrow \text{ if } \sigma \rightarrow \infty \text{ then } \sum_{i=1}^l \alpha_i y_i = 0 \\
& \text{Smoothness:} \\
& -\log \left( \int_{0}^{\infty} \hspace{-.1cm} B e^{-B \lambda} e^{\beta \lambda} \, d\lambda \hspace{-.05cm} \right) \hspace{-.05cm} = \log \left(1 - \frac{\beta}{B} \right) \hspace{-.05cm} \rightarrow 0 \text{ as } B \rightarrow \infty \\
& \text{Margin:} \\
& -\log \left( \int_{-\infty}^{1} C e^{-C(1-\gamma_i)} e^{- \alpha_i \gamma_i} \, d\gamma_i \right) = \alpha_i + \log(1 - \alpha_i / C) \\ 
& \text{Weight:} \\
& - \hspace{-.05cm} \log \left( \int_{-\infty}^\infty \hspace{-.1cm} \frac{e^{-\bm{\theta}^T \bm{\theta}/2}}{(2\pi)^{p/2}} \exp \hspace{-.05cm} \left\{ \sum_{i=1}^l \alpha_i y_i \bm{X}_i \bm{\theta} - \beta \bm{\theta}^T \bm{X}^T \mathcal{L} \bm{X} \bm{\theta} \hspace{-.05cm} \right\} \hspace{-.05cm} d\bm{\theta} \hspace{-.05cm} \right) \\
& = -\frac{1}{2} \left(\sum_{i=1}^l \alpha_i y_i \bm{X}_i \right) (\bm{I} + 2\beta \bm{X}^T \mathcal{L}\bm{X})^{-1} \left(\sum_{i=1}^l \bm{X}_i^T \alpha_i y_i \right) \\
& \quad \, + \log \left(\det(\bm{I} + 2\beta \bm{X}^T \mathcal{L}\bm{X}) \right) \\
& = -\frac{1}{2} \bm{\alpha}^T \bm{Y J} ( \bm{K}^{-1} + 2 \beta \mathcal{L})^{-1} \bm{J}^T \bm{Y \alpha} + \text{tr} \left(\log(\bm{I} + 2\beta \mathcal{L}\bm{K}) \right) \\
& \propto - \frac{1}{2} \bm{\alpha}^T \bm{Y J} \bm{K} ( \bm{I} + 2\beta \mathcal{L}\bm{K})^{-1} \bm{J}^T \bm{Y \alpha} .
\end{flalign*}
Thus the relationship between the probabilistic primal estimator and the kernel dual estimator is $ \bm{X} \hat{\bm{\theta}} = \bm{X} (\bm{I} + 2 \beta \bm{X}^T \bm{L X})^{-1} \bm{X}^T \bm{J}^T \bm{Y} \hat{\bm{\alpha}} = \bm{K} ( \bm{I} + 2\beta \mathcal{L}\bm{K})^{-1} \bm{J}^T \bm{Y} \hat{\bm{\alpha}} $.
\end{proof}

\section{Proofs for Section IV}

\begin{proof}[Proof of Lemma \ref{complete_post}]

Note that all labeled points are anomalous so $ y_i \eta_i = y_i $ for all $ i \in [1, l] $ or $ \bm{J H} = \bm{J}$. Thus following the same procedure as Proposition \ref{lapmed}, the MAP estimator for the posterior $ \P(\bm{\theta}, b, \bm{\gamma}, \lambda | \bm{X}, \bm{\eta}, \bm{y}) $ is $ \hat{\bm{\theta}} = ( \bm{I} + 2\beta \bm{X}^T \bm{h}^T \mathcal{L}_\eta \bm{h} \bm{X})^{-1} \bm{X}^T \bm{J}^T \bm{Y} \bm{\alpha} $ where the Lagrange multipliers $\bm{\alpha}$ are the solution to $ \underset{\bm{\alpha}}{\arg\max}-\log \left( Z(\bm{\alpha}) \right) $, which has terms
\begin{flalign*} 
& \text{Bias:} -\log \left(\int_{-\infty}^\infty \frac{e^{-b^2 / 2 \sigma^2 }}{2 \pi \sigma^2} \exp \left\{ \sum_{i=1}^l \alpha_i y_i \eta_i b \right\} db \right) \\
& = -\frac{\sigma^2}{2} \left(\sum_{i=1}^l \alpha_i y_i \eta_i \right)^2 \Rightarrow \text{ if } \sigma \rightarrow \infty, \\
& \hspace{115pt} \text{ then } \sum_{i=1}^l \alpha_i y_i \eta_i = \sum_{i=1}^l \alpha_i y_i = 0 \\
& \text{Smooth: Same as Proposition \ref{lapmed}} \\
& \text{Margin: Same as Proposition \ref{lapmed}} \\ 
& \text{Weight:} \\
& -\log \Bigg( \int_{-\infty}^\infty \hspace{-.1cm} \frac{e^{-\bm{\theta}^T \bm{\theta}/2}}{(2\pi)^{p/2}} \\
& \qquad \qquad \exp \left\{ \sum_{i=1}^l \alpha_i y_i \eta_i \bm{X}_i \bm{\theta} - \beta \bm{\theta}^T \bm{X}^T \bm{h}^T \mathcal{L}_\eta \bm{h} \bm{X} \bm{\theta} \right\} d\bm{\theta} \hspace{-.05cm} \Bigg) \\
& \propto \hspace{-.1cm} -\frac{1}{2} \hspace{-.07cm} \left(\sum_{i=1}^l \alpha_i y_i \eta_i \bm{X}_i \hspace{-.1cm} \right) \hspace{-.1cm} (\bm{I} + 2\beta \bm{X}^T \hspace{-.07cm} \bm{h}^T \hspace{-.07cm} \mathcal{L}_\eta \bm{h} \bm{X})^{-1} \hspace{-.13cm} \left(\sum_{i=1}^l \bm{X}_i^T \hspace{-.07cm} \alpha_i y_i \eta_i \hspace{-.1cm} \right) \\
& = -\frac{1}{2} \bm{\alpha}^T \bm{Y J} \bm{H} ( 2\beta \bm{h}^T \mathcal{L}_{\eta} \bm{h} + \bm{K}^{-1})^{-1} \bm{H} \bm{J}^T \bm{Y \alpha} \\
& = -\frac{1}{2} \bm{\alpha}^T \bm{Y J} \big( (2\beta \bm{h}^T \mathcal{L}_{\eta} \bm{h} + \bm{K}^{-1})^{-1} \odot \bm{\eta} \bm{\eta}^T \big) \bm{J}^T \bm{Y \alpha} 
\end{flalign*}
Instead of $ n \times n $ matrix operations, the Weight term can be compressed to $ a \times a $ matrix operations by permuting the rows and columns of the matrices so that the first $ a $ rows/cols correspond to $ \eta_i = 1 $. 
Then $ \left( (2\beta \bm{h}^T \mathcal{L}_{\eta} \bm{h} + \bm{K}^{-1})^{-1} \odot \bm{\eta} \bm{\eta}^T \right) $ is 0 everywhere except for the top left $ a \times a $ block, which (using block matrix inversion) can be expressed as $ [ (2\beta \bm{h}^T \mathcal{L}_{\eta} \bm{h} + \bm{K}^{-1})^{-1} ]_{11} = (\mathcal{L}_{\eta} + \bm{K}_{\eta}^{-1})^{-1} = \bm{K}_{\eta} (\bm{I} + \mathcal{L}_{\eta} \bm{K}_{\eta})^{-1} $ where $ \bm{K}_{\eta} $ is the top left $ a \times a $ block of the original Gram matrix.

Thus the primal dual relationship is $ \bm{X} \hat{\bm{\theta}} = \bm{K}_\eta ( \bm{I} + 2\beta \mathcal{L}_\eta \bm{K}_\eta)^{-1} \bm{J}^T \bm{Y} \hat{\bm{\alpha}} $.
\end{proof}

\begin{proof}[Proof of Lemma \ref{lowerbound}]
	
Because $ \P(\bm{X}, \bm{\eta}, \bm{y} | \bm{\alpha} ) $ is log convex, the function $ -\log(\cdot) $ is concave on its domain. Thus by Jensen's Inequality,
\begin{flalign*}
& -\log \left( \P(\bm{X}, \bm{y} | \bm{\alpha}) \right) = -\log \left( \sum_{\eta_1=0}^1 \dots \sum_{\eta_n=0}^1 q(\bm{\eta}) \frac{\P(\bm{X}, \bm{\eta}, \bm{y}| \bm{\alpha})}{q(\bm{\eta})} \right) \\
& \ge \text{E}_{\eta} \left( -\log \P(\bm{X}, \bm{y}, \bm{\eta} | \bm{\alpha}) \right) - \text{E}_{\eta} \left(-\log q(\bm{\eta}) \right)
\end{flalign*} 
where $ q(\bm{\eta}) $ is an arbitrary distribution. A natural choice for the distribution is $ q(\bm{\eta}) = \P(\bm{\eta} | \bm{X}, \bm{y}, \bm{\alpha}^{t-1}) $ where $ \bm{\alpha}^{t-1} $ are the optimal Lagrange multipliers of the previous iteration. Since the second term $ \text{E}_{\eta} \left(\log \left( \P(\bm{\eta} | \bm{X}, \bm{y}, \bm{\alpha}^{t-1}) \right) \right)$ does not depend on $ \bm{\alpha} $, it can be dropped so the lower bound is proportional to just the first term $ \E_\eta \left( -\log \left( Z(\bm{\alpha}) \right) \right) = $
\begin{flalign*}
& \sum_{i=1}^l \alpha_i + \E_\eta \left( -\frac{1}{2} \bm{\alpha}^T \bm{Y J} \bm{H} ( 2\beta \bm{H}^T \mathcal{L}\bm{H} + \bm{K}^{-1})^{-1} \bm{H} \bm{J}^T \bm{Y \alpha} \right) \\
& + \sum_{i=1}^l \log(1-\alpha_i / C) \,\, \text{ s.t. } \sum_{i=1}^l \alpha_i y_i = 0, \alpha_1, \dots, \alpha_l \ge 0.
\end{flalign*}
By Jensen's inequality again and ($\bm{J H } = \bm{J}$), the quadratic term above has lower bound
\begin{flalign*}
& \ge -\frac{1}{2} \bm{\alpha}^T \bm{Y J} \left( \bm{K}^{-1} + 2 \E_\eta ( \bm{H}^T \mathcal{L}\bm{H}) \right)^{-1} \bm{J}^T \bm{Y \alpha} \\
& = -\frac{1}{2} \bm{\alpha}^T \bm{Y J} \bm{K} \left( \bm{I} + 2 \E_\eta ( \bm{H}^T \mathcal{L}\bm{H}) \bm{K} \right)^{-1} \bm{J}^T \bm{Y \alpha}.
\end{flalign*}

\end{proof}

\begin{proof}[Proof of Lemma \ref{GEM_est}]
Define $k_L$ and $k_{L_\eta}$ as the number of neighbors in the kNNG of the graph Laplacians $\mathcal{L}$ and $\mathcal{L}_\eta$. There are kNNG with at least $k_L $ neighbors in $ \bm{H}^T \mathcal{L}\bm{H} $, which is formed on all the data and then pruned to just contain just the anomalous nodes, that will contain the subgraph in $\bm{h}^T \mathcal{L}_\eta \bm{h} $, which is a kNNG of only anomalous nodes with $ k_{L_\eta} $ neighbors. This is true for any $\bm{\eta} $ or its estimators $\hat{\bm{\eta}} $. So, there exists some $m$ (defined as the first $m$ points of any anomalous point are also anomalous) and $k_L \ge k_{L_\eta} $ such that
\begin{flalign*}
& || \E( \bm{H}^T \mathcal{L}\bm{H} | \bm{X}, \bm{y}, \bm{\alpha}^{t-1} ) - \E( \bm{h}^T \mathcal{L}_\eta \bm{h} | \bm{X}, \bm{y}, \bm{\alpha}^{t-1} ) ||_F \le \delta(m) \\
& || \hat{\bm{h}}^T \hat{\mathcal{L}}_\eta \hat{\bm{h}} - \hat{\bm{H}}^T \mathcal{L}\hat{\bm{H}} ||_F \le \delta'(m) 
\end{flalign*}
with equality and $\delta(m) =\delta'(m) = 0$ when $ k_L = k_{L_\eta} = m$ because then the pruned graph is exactly the graph in $\mathcal{L}_\eta $.

And since the GEM principle described in Section \ref{GEM} gives a good estimator, then 
\begin{flalign*}
& || \hat{\bm{H}}^T \mathcal{L}\hat{\bm{H}} - \E( \bm{H}^T \mathcal{L}\bm{H} | \bm{X}, \bm{y}, \bm{\alpha}^{t-1} ) ||_F \le \zeta
\end{flalign*}
has small $\zeta$.
So if $m$ is sufficiently large relative to $k_{L_\eta}$ so that $\delta'(m)$ and $\delta(m)$ are small, then $\hat{\bm{h}}^T \hat{\mathcal{L}}_\eta \hat{\bm{h}} $ is a good estimator
\begin{flalign*}
& || \hat{\bm{h}}^T \hat{\mathcal{L}}_\eta \hat{\bm{h}} - \E( \bm{h}^T \mathcal{L}\bm{h} | \bm{X}, \bm{y}, \bm{\alpha}^{t-1} ) ||_F \le || \hat{\bm{h}}^T \hat{\mathcal{L}}_\eta \hat{\bm{h}} - \hat{\bm{H}}^T \mathcal{L}\hat{\bm{H}} ||_F \\
& + || \hat{\bm{H}}^T \mathcal{L}\hat{\bm{H}} - \E( \bm{H}^T \mathcal{L}\bm{H} | \bm{X}, \bm{y}, \bm{\alpha}^{t-1} ) ||_F \\
& + || \E( \bm{H}^T \mathcal{L}\bm{H} | \bm{X}, \bm{y}, \bm{\alpha}^{t-1} ) - \E( \bm{h}^T \mathcal{L}_\eta \bm{h} | \bm{X}, \bm{y}, \bm{\alpha}^{t-1} ) ||_F \\
& \le \delta'(m) + \zeta + \delta(m)
\end{flalign*}
 (by triangle inequality) because $ \delta'(m) + \zeta + \delta(m) $ is also small.

\end{proof}

\begin{proof}[Proof of Theorem \ref{EM_alg}]
By Jensen's inequality, the log observed posterior has tight lower bound,
\begin{flalign*}
& \log \left( \P(\bm{\theta}, b, \bm{\gamma}, \lambda | \bm{X}, \bm{y}) \right) \ge \log \left( \P_0(\bm{\theta}, b, \bm{\gamma}, \lambda) \right) \\
& \quad + \E_\eta \left( \log \left( \P(\bm{X}, \bm{\eta}, \bm{y}| \bm{\theta}, b, \bm{\gamma}, \lambda) \right) \right) - \E_\eta \left( \log \left( \P(\bm{X}, \bm{\eta} , \bm{y} | \bm{\alpha}) \right) \right) 
\end{flalign*}
where the expectation is with respect to $ \P(\bm{\eta} | \bm{X}, \bm{y}, \bm{\alpha}^{t-1}) $. When the posterior is the MED solution using constraints \eqref{decision} and \eqref{smooth_h}, maximizing the lower bound for $ \bm{\theta} $ gives the primal form for the M-step as the solution to derivative of the lower bound
\begin{flalign*}
\sum_{i=1}^l \alpha_i y_i \E_\eta(\eta) \bm{X}^T_i - \left( \bm{I} + 2 \beta \bm{X}^T \E_\eta( \bm{h}^T \mathcal{L}_\eta \bm{h})\bm{X} \right) \bm{\theta} = \bm{0}.
\end{flalign*}

Following the same procedure as Lemma \ref{lowerbound}, the dual form for the M-step has a lower bound with quadratic term
\begin{flalign*}
& -\frac{1}{2} \bm{\alpha}^T \bm{Y J} \bm{K} \left( \bm{I} + 2 \E_\eta ( \bm{h}^T \mathcal{L}_\eta \bm{h}) \bm{K} \right)^{-1} \bm{J}^T \bm{Y \alpha}. 
\end{flalign*}
So using the same block matrix inversion procedure as Lemma \ref{complete_post}, the dual objective for the M-step is
\begin{flalign*}
& \sum_{i=1}^l \alpha_i + \log \left(1-\frac{\alpha_i }{ C} \right) -\frac{1}{2} \bm{\alpha}^T \bm{Y J} \hat{\bm{K}}_\eta ( \bm{I} + 2\beta \hat{\mathcal{L}}_{\eta} \hat{\bm{K}}_\eta )^{-1} \bm{J}^T \bm{Y} \bm{\alpha} \\
& \text{ s.t. } \sum_{i=1}^l \alpha_i y_i = 0, \,\, \alpha_1, \dots, \alpha_l \ge 0
\end{flalign*}
where $ \hat{\mathcal{L}}_{\eta} $ is the Laplacian matrix on only the set of data points $ \{ \bm{X}_i : \hat{\eta_i } = 1 \} $ and $ \hat{\bm{K}}_\eta $ is the $ a \times a $ submatrix of these same data points.
\end{proof}

\begin{proof}[Proof of Corollary \ref{decision_rule}]
The primal dual relationship is $ \bm{X} \hat{\bm{\theta}} = k(\bm{X}, \bm{X}_{\hat{\eta}}) ( \bm{I} + 2\beta \hat{\mathcal{L}}_\eta \hat{\bm{K}}_\eta)^{-1} \bm{J}^T \bm{Y} \hat{\bm{\alpha}} $. So for any point $\bm{X}_{i'}$, the prediction is $ \hat{\eta}_{i'} \left( \bm{X}_{i'} \hat{\bm{\theta}} + \hat{b} \right) = $
\begin{flalign*} 
& \hat{\eta}_{i'} \left(k(\bm{X}_{i'}, \bm{X}_{\hat{\eta}}) ( \bm{I} + 2\beta \hat{\mathcal{L}}_\eta \hat{\bm{K}}_\eta)^{-1} \bm{J}^T \bm{Y} \hat{\bm{\alpha}} + \hat{b} \right).
\end{flalign*}
 Because all nominal points are low utility, for simplicity they will be given predicted label $-1$.
\end{proof}
	
\section*{Acknowledgment}
This work was partially supported by the Consortium for Verification Technology under Department of Energy National Nuclear Security Administration award number DE-NA0002534, partially by the University of Michigan ECE Departmental Fellowship, and partially by the Xerox University Grant.

\ifCLASSOPTIONcaptionsoff
 \newpage
\fi



\bibliographystyle{IEEEtran}
\bibliography{references}

%

%
%
%




\end{document}